\documentclass{sig-alternate}

\usepackage{url}
\usepackage{bbm} % for indicator function
\usepackage[%
hyperfigures=true,%
backref=page,%
pagebackref=true,%
breaklinks=true,%
colorlinks=true,%
citecolor=green,%
linkcolor=blue%
]{hyperref}
\usepackage{amsmath}
\usepackage{amsfonts}
\usepackage{amssymb}
\usepackage{makeidx}
\usepackage{pifont}
\usepackage{graphicx, subfigure, epsf, latexsym, psfrag, bm}
\usepackage{graphics, array}
\usepackage{epsfig}
\usepackage{latexsym}
\usepackage[mathcal]{euscript}
\usepackage[vcentermath]{youngtab}
\usepackage{multirow}
\usepackage{citesort}
\usepackage{color}
\usepackage{pstricks}
\usepackage{cite}
\usepackage{multirow}
\usepackage{comment}
	\newtheorem{theorem}{Theorem}
	\newtheorem{lemma}{Lemma}
	\newtheorem{proposition}{Proposition}

	\newtheorem{definition}{Definition}

    \newcommand{\prl}[1] 		{\!\left(#1\right)}
    \newcommand{\brl}[1] 		{\left[ #1 \right] }
    \newcommand{\crl}[1] 		{\left\{ #1 \right\} }
    
    \newcommand{\R}             {\mathbb{R}} % Reals
 	\newcommand{\N}         		{\mathbb{N}} % Natural Numbers
 	\newcommand{\Sp}        		{\mathbb{S}} % Sphere
  	\newcommand{\card}[1]		{\left| #1 \right|} % cardinality
	\newcommand{\norm}[1] 		{\left \| #1 \right \|} 
    \newcommand{\absval}[1] 		{\left | #1 \right |}
    
    \newcommand{\dotprod}[2]    { #1 \cdot #2}

	\newcommand{\argmin}{\operatornamewithlimits{arg\ min}} % with space
	\newcommand{\argmax}{\operatornamewithlimits{arg\ max}} % with space

	\newcommand{\mat}[1] 		{\mathrm{\mathbf{#1}}}

	\newcommand{\vect}[1] 		{\mathrm{#1}}
	\newcommand{\vectbf}[1]		{\mathrm{\mathbf{#1}}}

    \newcommand{\refeqn}[1]			{(\ref{#1})}
    \newcommand{\reffig}[1]			{Figure \ref{#1}}
    \newcommand{\refsec}[1]			{Section \ref{#1}}
    \newcommand{\refapp}[1]			{Appendix \ref{#1}}
	\newcommand{\reftab}[1]			{Table \ref{#1}}
	\newcommand{\refthm}[1]			{Theorem \ref{#1}}
	\newcommand{\refprop}[1]		{Proposition \ref{#1}}
	
	\newcommand{\reflem}[1]			{Lemma \ref{#1}}
	\newcommand{\refdef}[1]			{Definition \ref{#1}}

    \newcommand{\ldf}				{  \: {\mathbf := }  \: }

    \newcommand{\sqz}[1]            {\! #1 \!}

\newcommand{\bigO}[1] 			{\mathrm{O}\prl{#1}}

\def \mathscr {\mathcal} % I don't seem to have the \mathscr series 

\newcommand{\bintreetopspace}     {\mathcal{BT}}
\newcommand{\tree}        		{\tau}
\newcommand{\treetop}			{\tau}
\newcommand{\treeA}              {\sigma}
\newcommand{\treeB}             {\tau}

\newcommand{\indexset}			{J}

\newcommand{\dist} 				{\mathit{d}}

\newcommand{\ctrd}[1]			{\mathrm{c}\prl{ #1}}
\newcommand{\var}[1]            {\mathrm{v} \prl{#1}}

\newcommand{\SSE}[1]               {\mathrm{SSE}\prl{#1}}
\newcommand{\cluster}[1] 		{\mathscr{C} \left( #1 \right ) }

\newcommand{\PowerSet}[1]       {\mathcal{P}\prl{#1}}

\newcommand{\parentCL}[1]		{Pr \left( #1 \right)}
\newcommand{\grandparentCL}[1]	{\mathrm{Pr}^2\prl{#1}}
\newcommand{\grandchildset}[1]   {\mathcal{G}\prl{#1}}
\newcommand{\childCL}[1]		    {\mathrm{Ch}\prl{#1}}

\newcommand{\ancestorCL}[1]     {Anc \left( #1\right )}

\newcommand{\complementCL}[1]   {{#1}^{C}}
\newcommand{\complementLCL}[1]  {{#1}^{LC}}
\newcommand{\compLCL}[2]        {{#1}^{- #2}}
\newcommand{\lv}  				{ \ell } % The level or ``depth'' of the partition in the tree
\newcommand{\depth}[1]          {\lv_{#1}}

\newcommand{\linkage}           {\zeta}

\newcommand{\HAC}              {\mathtt{HAC}}

\newcommand{\AHC}              {\mathtt{AHC}}
\newcommand{\IHC}              {\mathtt{IHC}}

 \newcommand{\NNI}              	{\text{NNI}}

\makeatletter
\def\@copyrightspace{\relax}

\def\ps@myheadings{
             \def\@oddhead{\hbox{}\sl\rightmark \hfil }% \def\@oddhead{\hbox{}\sl\rightmark \hfil\rm\thepage}
		     \def \@evenhead{ \hfil \sl\leftmark\hbox{}}% \def \@evenhead{\rm \thepage\hfil\sl\leftmark\hbox{}}
		     \def\sectionmark##1{}%
		     \def\subsectionmark##1{}     
		     }
\makeatother

\begin{document}
\pagestyle{myheadings}
\markboth{ESE Technical Report - \today}{ESE Technical Report - \today}
%
% --- Author Metadata here ---
%\conferenceinfo{KDD}{'14 New York City, New York USA}
%\CopyrightYear{2007} % Allows default copyright year (20XX) to be over-ridden - IF NEED BE.
%\crdata{0-12345-67-8/90/01}  % Allows default copyright data (0-89791-88-6/97/05) to be over-ridden - IF NEED BE.
% --- End of Author Metadata ---

\title{Anytime Hierarchical Clustering 
%\\ {\Large \textit{(Extended Version)}}
%\titlenote{(Produces the permission block, and
%copyright information). For use with
%SIG-ALTERNATE.CLS. Supported by ACM.}
}
%\subtitle{[Extended Abstract]
%\titlenote{A full version of this paper is available as
%\textit{Author's Guide to Preparing ACM SIG Proceedings Using
%\LaTeX$2_\epsilon$\ and BibTeX} at

\numberofauthors{2} %  in this sample file, there are a *total*
\author{
% 1st. author
\alignauthor
Omur Arslan\\
\affaddr{Department of Electrical and Systems Engineering, University of
Pennsylvania, Philadelphia, PA 19104}
\email{omur@seas.upenn.edu}
% 2nd. author
\alignauthor 
Daniel E. Koditschek \\
\affaddr{Department of Electrical and Systems Engineering, University of
Pennsylvania, Philadelphia, PA 19104}
\email{kod@seas.upenn.edu}
%
%% nth. author
%\alignauthor
%Ben Trovato\titlenote{Dr.~Trovato insisted his name be first.}\\
%       \affaddr{Institute for Clarity in Documentation}\\
%       \affaddr{1932 Wallamaloo Lane}\\
%       \affaddr{Wallamaloo, New Zealand}\\
%       \email{trovato@corporation.com}
}
%% There's nothing stopping you putting the seventh, eighth, etc.
%% author on the opening page (as the 'third row') but we ask,
%% for aesthetic reasons that you place these 'additional authors'
%% in the \additional authors block, viz.
%\additionalauthors{Additional authors: John Smith (The Th{\o}rv{\"a}ld Group,
%email: {\texttt{jsmith@affiliation.org}}) and Julius P.~Kumquat
%(The Kumquat Consortium, email: {\texttt{jpkumquat@consortium.net}}).}
% Just remember to make sure that the TOTAL number of authors
% is the number that will appear on the first page PLUS the
% number that will appear in the \additionalauthors section.
\maketitle
\begin{abstract}

We propose a new anytime hierarchical clustering method that iteratively transforms an arbitrary initial hierarchy on the configuration of measurements along a sequence of trees we prove for a fixed data set must terminate in a chain of nested partitions that satisfies a natural homogeneity requirement.  
Each recursive step re-edits the tree so as to improve a local measure of cluster homogeneity  that is compatible with a number of commonly used (e.g., single, average, complete) linkage functions.  
As an alternative to the standard batch algorithms, we present numerical evidence to suggest that appropriate adaptations of this method can yield decentralized, scalable algorithms suitable for distributed /parallel computation of clustering hierarchies and online tracking of clustering trees applicable to large, dynamically changing databases and anomaly detection. 

\end{abstract}
\category{H.3.3}{Information Storage and Retrieval}{Information Search and Retrieval}[Clustering]
%A category including the fourth, optional field follows...
\category{I.5.3}{Pattern Recognition}{Clustering}
\terms{Theory, Algorithms}
\keywords{Online Clustering, Homogeneity, Anytime Algorithm, Cluster Tracking, Nearest Neighbor Interchange, Big Data}

%%%%%%%%%%%%%%%%%%%%%%%%%%%%%%%%%%%%%%%%%%%%%%%%%%%%%%%%%
%%%%%%%%%%%%%%%%%%%%%%%%%%%%%%%%%%%%%%%%%%%%%%%%%%%%%%%%%
\section{Introduction}
\label{sec.Introduction}
%%%%%%%%%%%%%%%%%%%%%%%%%%%%%%%%%%%%%%%%%%%%%%%%%%%%%%%%%
%%%%%%%%%%%%%%%%%%%%%%%%%%%%%%%%%%%%%%%%%%%%%%%%%%%%%%%%%

The explosive growth of data sets in recent years is fueling a search for efficient and effective knowledge discovery tools. 
Cluster analysis \cite{berkhin_GMD2006, jain_dubes_1988,han_kamber_pei_2006} is a nearly ubiquitous  tool for unsupervised learning  aimed at discovering unknown groupings within a given set of unlabelled data points (objects, patterns), presuming that objects in the same group (cluster) are more similar to each other (intra-cluster homogeneity) than objects in other groups (inter-cluster separability).  
Amongst the many alternative methods, this paper focuses on dissimilarity-based hierarchical clustering as represented by a tree that indexes a chain of successively more finely nested partitions of the dataset. 
We are motivated to explore this approach to knowledge discovery because clustering can be imposed on arbitrary data types, and hierarchy can relieve the need to commit a priori to a specific partition cardinality  or granularity \cite{berkhin_GMD2006}.  
However, we are generally interested in online or reactive problem settings and in this regard hierarchical clustering suffers a number of long discussed weaknesses that become  particularly acute when scaling up to large, and, particularly, dynamically changing, data sets \cite{berkhin_GMD2006,han_kamber_pei_2006}. 
Construction of a clustering hierarchy (tree) generally requires $\bigO{n^2}$ time with the number
of data points $n$ \cite{olson_PC1995}.  
Moreover, whenever a data set is changed by insertion, deletion or update of a single data point, a clustering tree must generally be reconstructed in its entirety. 
This paper addresses the problem of   anytime online reclustering.

\subsection{Contributions of The Paper}
\label{sec.Contribution}
%%%%%%%%%%%%%%%%%%%%%%%%%%%%%%%%%%%%%%%%%%%%%
%%%%%%%%%%%%%%%%%%%%%%%%%%%%%%%%%%%%%%%%%%%%%

We introduce a new  homogeneity criterion  applicable to an array of  agglomerative (``bottom up") clustering methods through a test involving their  ``linkage function" --- the mechanism by which dissimilarity at the level of individual data entries is ``lifted" to the level of the clusters they are assigned. 
That criterion motivates  a ``homogenizing" local adjustment of the nesting relationship between proximal clusters in the hierarchy
that increases the degree of similitude within them while increasing the dissimilarity between them. 
We show that iterated application of this local homogenizing operation transforms any initial cluster hierarchy through  a succession of  increasingly ``better sorted" ones along a path in the abstract space of hierarchies that we prove, for a fixed data set and with respect to a family of linkages including the common single, average and complete cases, must converge in a finite number of steps. 
In particular, for the single linkage function, we prove convergence from any initial condition of any  sequence of arbitrarily chosen local homogenizing reassignments to the generically unique\footnote{In the generic case, all pairwise distances of data points are distinct and this guarantees that single linkage clustering yields a unique tree \cite{gallager_humblet_spira_TOPLAS1983,gower_ross_jrss1969}.}, globally homogeneous  hierarchy that would emerge from application of the standard,  one-step ``batch" single linkage based agglomerative clustering procedure. 

We present evidence to suggest that decentralized algorithms based upon this homogenizing transformation can  scale effectively for anytime online hierarchical clustering of large and dynamically changing data sets. 
Each local homogenizing adjustment entails computation over a proper subset of the entire dataset  --- and, for some linkages, merely  its sufficient statistics (e.g. mean, variance).
In these circumstances, given the sufficient statistics of a dataset, 
% Omur: where do we define the notion of 
% a sufficient statistic of a dataset?
such a restructuring decision at a node of a clustering hierarchy can be made in constant time (for further discussion see \refsec{sec.ComplexityAnalysis}).
Recursively defined (``anytime") algorithms such as this are naturally  suited to time varying data sets that arise  insertions, deletions or updates of a set of data points must be accommodated.  
Our particular  local restructuring method can also cope with time-varying dissimilarity measures or cluster linkage functions such as might result from the introduction of  learning aimed at increasing clustering accuracy \cite{xing_EtAl_NIPS2002}.

\subsection{A Brief Summary of Related Literature}
\label{sec.RelatedLiterature}
%%%%%%%%%%%%%%%%%%%%%%%%%%%%%%%%%%%%%%%%%%%%%%%%%%%%
%%%%%%%%%%%%%%%%%%%%%%%%%%%%%%%%%%%%%%%%%%%%%%%%%%%%

Two common approaches to remediating the limited scaling capacity and static nature of hierarchical clustering methods are data abstraction (summarization) and incremental clustering \cite{berkhin_GMD2006,jain_murty_flynn_CSUR1999}. 

%A widely recognized limitation of hierarchical clustering methods is scalability for large, or dynamically changing, datasets.
%The two common approaches for overcoming this issue are data abstraction (summarization) and incremental clustering \cite{berkhin_GMD2006,jain_murty_flynn_CSUR1999}.

Rather than improving algorithmic complexity of a specific clustering method, data abstraction aims to scale down a large data set with minimum loss of information for efficient clustering.
The large  literature on data abstraction includes (but is not limited to) such widely used methods as:  random sampling (e.g., CLARANS \cite{ng_han_TOKDE2002}); selection of representative points (e.g., CURE \cite{guha_rastogi_shim_IS2001}, data bubble \cite{breunig_kriegel_kroger_sander_SIGMOD2001}); usage of cluster prototypes(e.g., Stream \cite{guhaEtAl_TKDE2003}) and sufficient statistics (e.g., BIRCH \cite{zhang_ramakrishnan_livny_SIGMOD1996}, scalable $k$-means \cite{bradley_fayyad_reina_KDD1998}, CluStream \cite{aggarwalEtAl_VLDB2003}, data squashing \cite{dumouchel_EtAl_KDD1999}); grid-based quantization  \cite{berkhin_GMD2006,han_kamber_pei_2006} and sparcification of connectivity or distance matrix (e.g., CHAMELEON \cite{karypis_han_kumar_C1999}).

In contrast, incremental approaches to hierarchical clustering generally target algorithmic improvements for efficient handling of large data sets by processing data in sequence, point by point.
Typically, incremental clustering proceeds in two stages: first (i) locate a new data point in the currently available clustering hierarchy, and then (ii) perform a set of restructuring operations (cluster merging, splitting or creation), based on a heuristic criterion, to obtain a better
clustering model. 
Unfortunately, this sequential process generally incurs unwelcome sensitivity to the order of presentation \cite{berkhin_GMD2006,jain_murty_flynn_CSUR1999}. 
Independent of the efficiency and accuracy of our clustering method, the results we report here may be of interest to those seeking  insight into the possible spread of outcomes across the combinatorial explosion of different paths through even a fixed data set.
Among the many alternatives (e.g., the widely accepted COBWEB \cite{fisher_ML1987} or BIRCH \cite{zhang_ramakrishnan_livny_SIGMOD1996} procedures), our anytime method most closely resembles the  incremental 
clustering approach of \cite{widyantoro_ioerger_yen_ICDM2002}, and relies on analogous structural criteria, using similar concepts (``homogeneity" and ``monotonicity").  
However, a major advantage afforded by our new homogeneity criterion (\refdef{def.LocalStructuralOptimality}) relative to that of \cite{widyantoro_ioerger_yen_ICDM2002} is that there is now no requirement for a minimum spanning tree over the dataset. 
Beyond ameliorating the computational burden, this relaxation  extends the applicability of our method beyond single-linkage to a subclass of linkages,  \refdef{def.Contraction},  a family of cluster distance functions that includes single, complete, average,
minimax and Ward's linkages \cite{jain_dubes_1988}. 

Of course, recursive (``anytime") methods  can be adapted to address  the  general setting of time varying data processing. 
Beyond the specifics of the data insertion problem handled  by incremental  clustering methods adapting,  we aim for reactive algorithms suited to a range of dynamic settings, including  data insertion, deletion, update or perhaps, a processing-induced  non-stationarity such time varying dissimilarity measure  or linkage function \cite{achtert_bohm_kriegel_kroger_ICDM2005}. 
Hence, as described in the previous section, we propose a partially decentralized, recursive method: a local cluster restructuring operation yielding a discrete dynamical system in the abstract space of trees guaranteed to improve the hierarchy at each step (relative to a fixed dataset) and to terminate in an appropriately homogenizing cluster hierarchy from any (perhaps even random) initial such structure.

\subsection{Organization of The Paper}
\label{sec.Organization}
%%%%%%%%%%%%%%%%%%%%%%%%%%%%%%%%%%%%%%%%
%%%%%%%%%%%%%%%%%%%%%%%%%%%%%%%%%%%%%%%%

\refsec{sec.Background} introduces notation and offers a brief
summary of the essential background. 
\refsec{sec.LocalStructuralOptimality} presents  our homogeneity criterion and establishes some of its properties. 
\refsec{sec.OnlineHierarchicalClustering} introduces a simple anytime hierarchical clustering method that seeks successively to ``homogenize" local clusters according to this criterion.   
We analyze the termination and complexity properties of the method and then illustrate its algorithmic implications by applying  it to the specific problem of incremental clustering.
\refsec{sec.ExperimentalEvaluation} presents experimental evaluation of the anytime hierarchical clustering method using both synthetic  and real datasets. 
We conclude with a brief discussion of future work in \refsec{sec.Conclusion}.

%%%%%%%%%%%%%%%%%%%%%%%%%%%%%%%%%%%%%%%%%%%%%%%
%%%%%%%%%%%%%%%%%%%%%%%%%%%%%%%%%%%%%%%%%%%%%%%     
\section{Background \& Notation}
\label{sec.Background}
%%%%%%%%%%%%%%%%%%%%%%%%%%%%%%%%%%%%%%%%%%%%%%%
%%%%%%%%%%%%%%%%%%%%%%%%%%%%%%%%%%%%%%%%%%%%%%%

%In this section, we  first recall a set theoretical formulation of hierarchies (rooted trees), hierarchical relations of tree clusters and a certain type of tree rearrangement.
%Then,  we shall give an overview of the standard notion of hierarchical agglomerative clustering and cluster merging criteria, \emph{linkages}. 

%%%%%%%%%%%%%%%%%%%%%%%%%%%%%%%%%%%%%%%%%%
%%%%%%%%%%%%%%%%%%%%%%%%%%%%%%%%%%%%%%%%%%
\subsection{Datasets, Patterns, and Statistics}
\label{sec.Pattern}
%%%%%%%%%%%%%%%%%%%%%%%%%%%%%%%%%%%%%%%%%%
%%%%%%%%%%%%%%%%%%%%%%%%%%%%%%%%%%%%%%%%%%

We consider data points (patterns, observations) in $\R^m$ with a dissimilarity measure\footnote{A dissimilarity measure $\dist : X \times X \rightarrow \R_{\geq0}$ in $X$ is a symmetric, non-negative and reflexive function, i.e.  $\dist\prl{x,y} =  \dist\prl{y,x}$, $\dist\prl{x, y} \geq 0$ and $\dist\prl{x,x} = 0$ for all $x,y \in X$. }
 $\dist : \R^m \times \R^m \rightarrow \R_{\geq 0}$, where $m \in \N$ is the dimension of the space containing the dataset and $\R_{\geq 0}$ denotes the set of non-negative real numbers.
Note that $\dist$ need not necessarily be a metric\footnote{A dissimilarity $\dist:X \times X \rightarrow \R_{\geq 0}$ is a metric if it satisfies strong reflexivity and the  triangle inequality, i.e. for all $x, y, z \in X$ $\dist\prl{x,y} = 0 \Longleftrightarrow x = y$ and $\dist\prl{x,y} \leq \dist\prl{x,z} + \dist\prl{z,y}$.}, and our results can be easily generalized to qualitative data as well, once some  dissimilarity ordering has been defined. 
 %\footnote{
%Note that our results  simply generalizes to observations whose dissimilarity encoded by an undirected graph, $G = \prl{\indexset, E, \mat{A}}$, with no self-cycles (i.e. its adjacency matrix is zero-diagonal and symmetric.) and non-negative weights,    $\mat{A} = \tr{\mat{A}}$ and $A_{jj} = 0$ for all $j \in \indexset$.}
 
Let $\vectbf{x} = \prl{\vect{x}_{j}}_{j \in \indexset} \in \prl{\R^m}^{\indexset}$ be a set of data points  bijectively labelled by a fixed finite index set $\indexset$, say $\indexset = \brl{n} \ldf \crl{1,2, \ldots, n}$, 
and let $\vectbf{x}|I = \prl{\vect{x}_i}_{i \in I}$ denote a partial set of  observations associated with subset $I \subseteq \indexset$, whose centroid and variance, respectively,  are %\footnote{Here, $\card{I}$ and $\norm{\vect{x}}_2$ denote the cardinality of set $I$ and the standard Euclidean norm of point $\vect{x}$, respectively.}
\begin{align}
\ctrd{\vectbf{x}|I} &\ldf \frac{1}{\card{I}} \sum_{i \in I} \vect{x}_i, \label{eq.centroid} \\ 
\var{\vectbf{x}|I} &\ldf \frac{1}{\card{I}} \sum_{i \in I} \norm{\vect{x}_i - \ctrd{\vectbf{x}|I}}_2^2,  \label{eq.variance} 
\end{align}    
where $\card{I}$ and $\norm{\vect{x}}_2$ denote the cardinality of set $I$ and the standard Euclidean norm of point $\vect{x} \in \R^{m}$, respectively.
Throughout the sequel the term ``sufficient cluster statistics" denotes the cardinality, $\card{I}$, and mean \refeqn{eq.centroid} and variance \refeqn{eq.variance} for each cluster in a hierarchy \cite{bradley_fayyad_reina_KDD1998}.

%%%%%%%%%%%%%%%%%%%%%%%%%%%%%%%%%%%%%%%%%%%%%%%%%%%%%
%%%%%%%%%%%%%%%%%%%%%%%%%%%%%%%%%%%%%%%%%%%%%%%%%%%%%
\subsection{Hierarchies}
\label{sec.Hierarchies}
%%%%%%%%%%%%%%%%%%%%%%%%%%%%%%%%%%%%%%%%%%%%%%%%%%%%%
%%%%%%%%%%%%%%%%%%%%%%%%%%%%%%%%%%%%%%%%%%%%%%%%%%%%%

A rooted semi-labelled tree $\tree$ over a fixed finite index set $\indexset$, illustrated in \reffig{fig:hierarchicalrelation}, is  a directed acyclic graph $G_{\tree} = \prl{V_{\tree}, E_{\tree} }$,
whose leaves, vertices of degree one, are bijectively labeled by $\indexset$ and interior vertices all have out-degree at least two; and all of whose edges in $E_{\tree}$ are directed away from a vertex designated to be the \emph{root} \cite{billera_holmes_vogtmann_aap2001}. 
A rooted semi-labelled tree $\tree$ uniquely determines (and henceforth will be interchangeably used with) a \emph{cluster hierarchy} \cite{mirkin_1996}. 
By definition, all vertices of $\tree$  can be reached from the root through a directed path in $\tree$. 
The \emph{cluster} of a vertex $v\in V_{\treetop}$ is defined to be the set of leaves reachable from $v$ by a directed path in $\tree$. 
%Let $\cluster{\tree}$ denote the set of all vertex clusters of $\treetop$.
%Correspondingly, the cluster set $\cluster{\tree}$ of $\tree$ is defined to be the set of all its vertex clusters.
% $\cluster{v}$, $v\in V_{\tree}$. 
Accordingly, the cluster set $\cluster{\tree}$ of $\tree$ is defined to be the set of all its vertex clusters,
\begin{align}
\cluster{\tree} \ldf \crl{\cluster{v} \big| v \in V_{\tree}} \subseteq \PowerSet{\indexset},
\end{align} 
where $\PowerSet{\indexset}$ denotes the power set of $\indexset$.
%Let $\cluster{\tree}$ denote the set of all the $\cluster{v}$, $v\in V_{\tree}$. 
%Note that $\cluster{\tree}$ is also known as the laminar family associated with $\tree$  \cite{schrijver2003}. 

\begin{figure}[tb]
%\centering
\hspace{1mm} \includegraphics[width=0.48\textwidth]{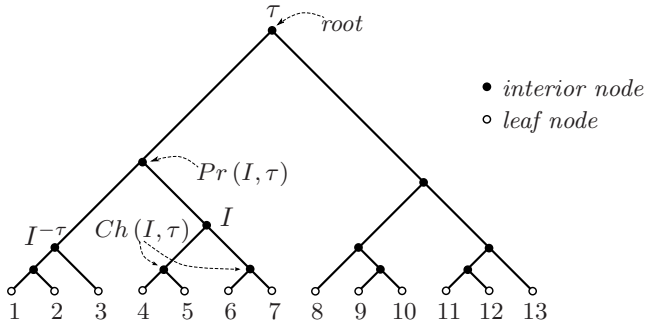} 
\caption{Hierarchical relations: parent - $\parentCL{I,\treetop}$, children - $\childCL{I,\treetop}$, and local complement (sibling) - $\compLCL{I}{\treetop}$ of cluster $I$ of a rooted binary  tree, $\treetop \in \bintreetopspace_{\brl{13}}$. 
Filled and unfilled circles represent interior and leaf nodes, respectively. An interior node is referred by its cluster, the list of leaves below it; for example, $I = \crl{4,5,6,7}$.%, $\parentCL{I,\treetop} = \crl{1,2,3,4,5,6,7}$, $\childCL{I,\treetop} = \crl{\big.\crl{4,5}, \crl{6,7}}$ and $\compLCL{I}{\treetop} = \crl{1,2,3}$. 
}
\label{fig:hierarchicalrelation}
\end{figure}

For every cluster $I \in \cluster{\tree}$ we recall the standard notion of \emph{parent} (cluster) $\parentCL{I,\tree}$ and lists of \emph{children} $\childCL{I,\tree}$ of $I$ in $\tree$, illustrated in \reffig{fig:hierarchicalrelation}. 
For the trivial case, we set $\parentCL{\indexset, \tree} = \emptyset$.
Additionally, we find it useful to define the \emph{local complement (sibling)} of cluster $I \in \cluster{\treetop}$ as $\compLCL{I}{\tree} \ldf \parentCL{I,\tree} \setminus I$, not to confused with the standard (global) complement $\complementCL{I} = \indexset \setminus I$.
%We will use the term \emph{uncle} of a cluster $I \in \cluster{\treetop}$ to refer the sibling of its parent $\parentCL{I,\treetop}$.
Further, a grandchild in $\tree$ is a cluster $G \in 
\cluster{\tree}$ having a grandparent $\grandparentCL{G,\tree} \ldf$ $\parentCL{\big. \parentCL{G,\tree}, \tree}$ in $\tree$. 
We denote the set of all grandchildren  in $\tree$  by $\grandchildset{\tree}$, the maximal subset of $\cluster{\treetop}$ excluding the root $\indexset$ and its children $\childCL{\indexset,\treetop}$,
\begin{subequations} \label{eq.grandchild}
\begin{align} 
\grandchildset{\tree} &\ldf \crl{G \in \cluster{\tree} \big | \grandparentCL{G,\tree} \neq \emptyset}, \\
& = \cluster{\tree} \setminus \prl{ \big. \crl{\indexset} \cup \childCL{\indexset,\tree}}.
\end{align} 
\end{subequations}

A rooted tree with all interior vertices of out-degree two is said to be { \em binary } or, equivalently, { \em non-degenerate}, and  all other trees are
said to be \emph{degenerate}.  
In this paper $\bintreetopspace_{\indexset}$ denotes the set of rooted nondegenerate trees over leaf set $\indexset$.
Note that the number of  hierarchies in $\bintreetopspace_{\indexset}$  grows super exponentially \cite{billera_holmes_vogtmann_aap2001},  
\begin{align}
\card{\bintreetopspace_{\indexset}} = \prl{\big. 2\card{\indexset} - 3} !! = \prl{ \big. 2\card{\indexset} - 3} \prl{\big.2\card{\indexset} - 5} \ldots 3,  
\label{eq.numBinTree}
\end{align}
for  $\card{\indexset}\geq 2$, quickly precluding the possibility of exhaustive search for the ``best" hierarchical clustering model in even modest problem settings.

%%%%%%%%%%%%%%%%%%%%%%%%%%%%%%%%%%%%%%%%%%%%%%%%%%%%%%%
%%%%%%%%%%%%%%%%%%%%%%%%%%%%%%%%%%%%%%%%%%%%%%%%%%%%%%%
\subsection{Nearest Neighbor Interchange (NNI) Moves}
\label{sec.NNIMoves}
%%%%%%%%%%%%%%%%%%%%%%%%%%%%%%%%%%%%%%%%%%%%%%%%%%%%%%%
%%%%%%%%%%%%%%%%%%%%%%%%%%%%%%%%%%%%%%%%%%%%%%%%%%%%%%%

Different notions of the neighborhood of
a non-degenerate hierarchy in $\bintreetopspace_{\indexset}$ can be imposed by recourse to different tree restructuring operations \cite{felsenstein2004} (or \emph{moves}). 
NNI moves are particularly important for our setting because of their close relation with cluster hierarchy homogeneity (\refdef{def.LocalStructuralOptimality}) and their role in the anytime procedure introduced in  \refsec{sec.OnlineHierarchicalClustering}. 

A convenient restatement of the standard definition of NNI walks \cite{robinson_jct1971, moore_goodman_barnabas_jtb1973} for rooted trees, illustrated in \reffig{fig.NNImove}, is: 
\begin{definition} \label{def:NNIMove}
The {\em Nearest Neighbor Interchange (NNI)} move at a grandchild $G \in \grandchildset{\treeA}$ on a binary hierarchy $\treeA \in \bintreetopspace_{\indexset}$ swaps cluster $G$ with its parent's sibling $\compLCL{\parentCL{G,\treeA}}{\treeA}$ to yield another binary hierarchy $\treeB \in \bintreetopspace_{\indexset}$. 

%We say that $\treeB \in \bintreetopspace_{\indexset}$ is the result of performing a  \emph{Nearest Neighbor Interchange (NNI)} move  on $\treeA \in \bintreetopspace_{\indexset}$  if there exists a grandchild  $G \in \grandchildset{\treeA}$  such that %$\cluster{\treeB}$ satisfies the following 
%%    
%\begin{align}
%\cluster{\treeB} &= \prl{ \big. \cluster{\treeA} \setminus \crl{\parentCL{G,\treeA}}} \cup \crl{\big.\compLCL{I}{\treeA} \cup \compLCL{G}{\treeA}}. \label{eq.NNImove}
%\end{align} 

We say that $\treeA,\treeB \in \bintreetopspace_{\indexset}$ are \emph{NNI-adjacent} if and only if one can be obtained from the other by a single NNI move.
\end{definition}
%
%Informally, the NNI move at a grandchild $G \in \grandchildset{\treeA}$ on $\treeA$ swaps  grandchild $G$ with its parent's sibling $ \compLCL{\parentCL{G,\treeA}}{\treeA}$.
More precisely,  $\treeB \in \bintreetopspace_{\indexset}$ is the result of performing the NNI move  at grandchild  $G \in \grandchildset{\treeA}$ on $\treeA \in \bintreetopspace_{\indexset}$   if 
\begin{align}
\cluster{\treeB} &= \prl{ \big. \cluster{\treeA} \setminus \crl{\parentCL{G,\treeA}}} \cup \crl{\big.\compLCL{I}{\treeA} \cup \compLCL{G}{\treeA}}. \label{eq.NNImove}
\end{align} 
Throughout the sequel we will denote the map of $\bintreetopspace_{\indexset}$  into itself defining an NNI move at a grandchild cluster of a tree $G \in \grandchildset{\treeA}$   as
$\treeB = \NNI\prl{\treeA, G}$. 

\begin{figure}[t]
\centering
\vspace{-2mm}
\includegraphics[width=0.35\textwidth]{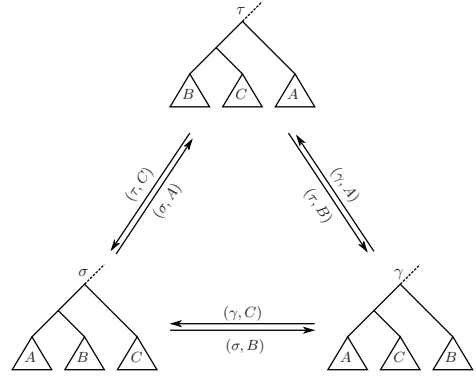} 
\vspace{-2mm}
\caption{An illustration of NNI moves between binary trees: each arrow is labeled by a source tree and associated cluster defining the move.}
\vspace{-3mm}
\label{fig.NNImove}
\end{figure}

%For any pair of  NNI-adjacent hierarchies $\prl{\treeA, \treeB}$, we may apply Lemma 1 from \cite{ArslanEtAl_NNITechReport2013} to find common disjoint clusters $A,B,C$ such that $\crl{A \cup B} \in \cluster{\treeA} \setminus \cluster{\treeB}$ and $\crl{B \cup C} = \cluster{\treeB} \setminus \cluster{\treeA}$.
%Note that the triplet $\prl{A,B,C}$ of $\prl{\treeA, \treeB}$ uniquely characterizes their  structural difference. 
%We call $\prl{A,B,C}$ the ``\emph{NNI-triplet}" of  $\;\prl{\treeA,\treeB}$.

\smallskip

A useful observation for NNI-adjacent hierarchies illustrating their structural difference is:
\begin{lemma}{\cite{ArslanEtAl_NNITechReport2013}} \label{lem.NNITriplet}
An ordered pair of hierarchies $\prl{\treeA,\treeB}$ is NNI-adjacent
if and only if there exists one and only one ordered triple $\prl{A,B,C}$ of common clusters of $\treeA$ and $\treeB$ such that $\crl{A \cup B} = \cluster{\treeA} \setminus \cluster{\treeB}$ and $\crl{B \cup C} = \cluster{\treeB} \setminus \cluster{\treeA}$.
%Call $\prl{A,B,C}$ the ``\emph{NNI-triplet}" of  $\;\prl{\treeA,\treeB}$.
\end{lemma}
We call $\prl{A,B,C}$ the ``\emph{NNI-triplet}" of  $\;\prl{\treeA,\treeB}$.

%%%%%%%%%%%%%%%%%%%%%%%%%%%%%%%%%%%%%%%%%%%
%%%%%%%%%%%%%%%%%%%%%%%%%%%%%%%%%%%%%%%%%%%
\subsection{Hierarchical Agglomerative Clustering}
\label{sec.HAC}
%%%%%%%%%%%%%%%%%%%%%%%%%%%%%%%%%%%%%%%%%%%
%%%%%%%%%%%%%%%%%%%%%%%%%%%%%%%%%%%%%%%%%%%

%As the readers shall notice  the close relation between  the notion of structural homogeneity (\refdef{def.LocalStructuralOptimality})  and  the general setting of agglomerative clustering methods, we find  it  useful to give a formal description of hierarchical agglomerative clustering.

%A  hierarchical agglomerative clustering (HAC) method is a bottom-up procedure that starts with the finest partition of a dataset where each cluster is a singleton, and at every stage, based on a certain criterion, merges a set of  clusters until it reaches the coarsest cluster which contains all data points \cite{jain_dubes_1988}. 
Given a choice of linkage, $\linkage:  \prl{\R^m}^{\indexset} \times \PowerSet{\indexset} \times \PowerSet{\indexset} \rightarrow \R_{\geq0}$, \reftab{tab.HAC} formalizes the associated Hierarchical Agglomerative $\linkage$-Clustering method \cite{jain_dubes_1988}. 
This method yields a sequence of nested partitions of the dataset that can be represented by a tree with root at the coarsest, single
cluster partition, leaves at the most refined trivial partition (comprising all singleton sets), and a vertex for each subset appearing in any of the successively coarsened partitions that appears as the mergers of \reftab{tab.HAC} are imposed. 
Because only two clusters are merged at each step, the resulting sequence of nested partitions defines a binary tree, $\treetop \in \bintreetopspace_{\indexset}$ whose nodes represent the clusters
\begin{align}
\cluster{\treetop} = \bigcup_{k= 0}^{\card{\indexset}-1} \mathcal{J}_k = \crl{\indexset} \cup \bigcup_{k=0}^{\card{\indexset} -2} \crl{A^*_k, B^*_k}. \label{eq.Partition2Cluster}
\end{align}  
and whose edges represent the nesting relation, again as presented in \reftab{tab.HAC}.  
Hence, the set of grandchildren clusters $\grandchildset{\treetop}$ \refeqn{eq.grandchild} of $\treetop$ is given by
\begin{align}
\grandchildset{\treetop} = \bigcup_{k= 0}^{\card{J}-3} \mathcal{J}_k = \bigcup_{k=0}^{\card{\indexset} -3} \crl{A^*_k, B^*_k}. \label{eq.Partition2GrandChild}
\end{align} 
From this discussion it is clear that \reftab{tab.HAC} defines a relation from datasets to trees, $\HAC_{\linkage}  \subset \prl{\R^m}^{\indexset} \times \bintreetopspace_{\indexset}$. Note that $\HAC_{\linkage}$ is in general not a function  since there may well be more than one pair of clusters satisfying \refeqn{eq.LinkageOptimal} at any  stage, $k$.
It is,  however, a multi-function: in other words, while agglomerative clustering of a dataset always yields some tree, that tree is not necessarily unique to that dataset.

\begin{table}[tb]
\caption{ \hspace{-2.75mm}Hierarchical  Agglomerative  $\linkage$-Clustering\!\cite{jain_dubes_1988}
}
\label{tab.HAC}
\centering
\begin{tabular}{|p{0.45\textwidth}|}
\hline
\\[-3mm]
%A formal restatement of the standard form of a hierarchical  agglomerative clustering method, denoted by $\HAC_{\linkage}$,  based on a particular linkage $\linkage : \prl{\R^m}^{\card{\indexset}} \times \PowerSet{\indexset} \times \PowerSet{\indexset} \rightarrow \R_{\geq0}$ is as follows.%\footnote{$\PowerSet{.}$ denotes the power set of its operand. }
For any given set of data points $\vectbf{x} \in \prl{\R^m}^{\indexset}$ and linkage function\footnotemark \;
  $\linkage : \prl{\R^m}^{\indexset} \times \PowerSet{\indexset} \times \PowerSet{\indexset} \rightarrow \R_{\geq0}$,

\begin{itemize}
\item Begin with the finest partition of $\indexset$, $\mathcal{J}_0 = \crl{\big. \crl{i}}_{i \in \indexset}$.
\item  For every $k \in [ 0, \card{J} - 1)$,
merge two blocks of $\mathcal{J}_{k}$ with the minimum linkage value,\footnotemark%
\begin{subequations}
\label{eq.HAC}
\begin{align}
\prl{A^*_k, B^*_k}  =  \argmin_{A\neq B \in \mathcal{J}_k} \linkage\prl{\vectbf{x}; A, B}, \hspace{12mm} \label{eq.LinkageOptimal}\\
\mathcal{J}_{k+1} = \crl{A^*_k \cup B^*_k} \cup \mathcal{J}_k \setminus \crl{A^*_k, B^*_k}.  \label{eq.LinkageUpdate}
\end{align}
\end{subequations}

\vspace{-5mm}

\end{itemize} 
\\
\hline 

\end{tabular}
\end{table}
\addtocounter{footnote}{-1}
\footnotetext{Note that the linkage between any partial observations and the empty set is always defined to be zero, i.e. $\linkage\prl{\vectbf{x}; I, \emptyset}= \linkage\prl{\vectbf{x};  \emptyset, I} = 0$ for all $I \subseteq \indexset$ and $\vectbf{x} \in \prl{\R^m}^{\indexset}$. } 
\addtocounter{footnote}{1}
\footnotetext{Note that a non-degenerate hierarcy over the leaf set $\indexset$ always has $\card{\indexset}- 1$  interior nodes \cite{robinson_jct1971}.}

%Also notice that the resultant hierarchy is not necessarily unique for a given data set  since there might be more than one pair of clusters satisfying \refeqn{eq.LinkageOptimal} at any stage of agglomerative clustering. %$\HAC_{\linkage}$.%, which becomes critical when we define the hierarchical strata for linkage clusterings and their neighboring relation. 

%%%

%%%%%%%%%%%%%%%%%%%%%%%%%%%%%%%%%%%%%%%%%%
%%%%%%%%%%%%%%%%%%%%%%%%%%%%%%%%%%%%%%%%%%
\subsubsection{Linkages}
\label{sec.linkage}
%%%%%%%%%%%%%%%%%%%%%%%%%%%%%%%%%%%%%%%%%%
%%%%%%%%%%%%%%%%%%%%%%%%%%%%%%%%%%%%%%%%%% 

%The agglomerating criterion, generally known as a \emph{linkage} (function), plays the role of lifting to the level of clusters the dissimilarity function defined on the level of data points \cite{ackermanEtAl_COLT2010}.

%The main difference between various agglomerative clustering methods is how they define  cluster (dis)similarities, i.e.  merging criteria (costs), which are generally referred to be \emph{linkage criteria (functions)} or, shortly, \emph{linkages}. 

%For any $A,B \subseteq \indexset$ and $\vectbf{x} \in \prl{\R^m}^{\card{\indexset}}$, we recall the standard linkage definitions describing  correlation between partial sets of observations $\vectbf{x}|A$ and $\vectbf{x}|B$, 

A  linkage, $\linkage: \prl{\R^m}^{\indexset} \times \PowerSet{\indexset} \times \PowerSet{\indexset} \rightarrow \R_{\geq 0}$,  uses the dissimilarity of observations in the partial datasets, $\vectbf{x}|A$ and $\vectbf{x}|B$, to define dissimilarity between the  clusters,  $A,B \subseteq J$ and  $\vectbf{x} \in \prl{\R^m}^{\card{\indexset}}$  \cite{ackermanEtAl_COLT2010}. 
Some common examples are
{\small
\begin{subequations} \label{eq.Linkage}
\begin{align}
\linkage_{S}\prl{\vectbf{x};A,B} & \ldf \min_{\substack{a \in A \\ b \in B}} \dist\prl{\vect{x}_a , \vect{x}_b}, \label{eq.SingleLinkage} \\
\linkage_{C}\prl{\vectbf{x};A,B} &\ldf \max_{\substack{a \in A \\ b \in B}} \dist\prl{\vect{x}_a , \vect{x}_b}, \label{eq.CompleteLinkage} \\
\linkage_{A}\prl{\vectbf{x};A,B} &\ldf \frac{1}{\card{A}\card{B}}\sum_{\substack{a \in A \\ b \in B}} \dist\prl{\vect{x}_a , \vect{x}_b}, \label{eq.AverageLinkage} \\
\linkage_{M}\prl{\vectbf{x};A,B} &\ldf \min_{a \in A \cup B} \max_{b \in A \cup B} \dist\prl{\vect{x}_a, \vect{x}_b}, 
\label{eq.MinimaxLinkage} \\
\linkage_{W}\prl{\vectbf{x};A,B} &\ldf \frac{\card{A} \card{B}}{\card{A} + \card{B}} \norm{\ctrd{\vectbf{x}|A} - \ctrd{\vectbf{x}|B}}_2^2 \label{eq.WardLinkage}
\end{align}
\end{subequations}
}
for single, complete, average, minimax and Ward's linkages, respectively \cite{jain_dubes_1988, bien_tibshirano_JASA2011}, where $\dist$ and  $\norm{.}_2$ are a dissimilarity measure in $\R^m$ and the standard Euclidean norm , respectively.
%Additionally, we find it useful to define the total linkage criterion to be 
%
%\addtocounter{equation}{-1}
%{\small
%\begin{subequations}
%\addtocounter{equation}{4}
%%\renewcommand{\theequation}{\theparentequation \alph{equation}}
%\begin{align}
%\linkage_{T}\prl{\vectbf{x};A,B} \ldf \sum_{\substack{a \in A\\ b \in B}} \dist_{s}\prl{\vect{x}_a, \vect{x}_b} = \card{A}\card{B} \linkage_{A}\prl{\vectbf{x};A,B}. \label{eq.TotalLinkage}
%\end{align} 
%\end{subequations}
%}

\smallskip

A common way  of characterizing linkages is through their behaviours after merging a set of clusters. 
For any pairwise disjoint subsets $A,B,C$ of $\indexset$ and dataset $\vectbf{x} \in \prl{\R^m}^{\indexset}$, a linkage relation between partial observations $\vectbf{x}|A\cup B$ and $\vectbf{x}|C$  after merging $\vectbf{x}|A$ and $\vectbf{x}|B$ %are given by \cite{murtagh_CJ1983}
are generally described by the recurrence formula of Lance and Williams \cite{lance_williams_CJ1967}, 
{
\begin{align} 
\linkage\prl{\vectbf{x}; \!A \sqz{\cup} B, \!C} &= \alpha_A \linkage\prl{\vectbf{x}; \!A, \!C} + \alpha_{B} \linkage\prl{\vectbf{x}; \!B, \!C} + \beta \linkage\prl{\vectbf{x};A, B} \nonumber \\
& \hspace{10mm}+ \gamma \absval{\big.\linkage\prl{\vectbf{x}; \!A , \!C} - \linkage\prl{\vectbf{x}; \!B, \!C}}, \label{eq.Recurrence}
\end{align} 
}
where $\linkage:\prl{\R^m}^{\indexset} \times \PowerSet{\indexset} \times \PowerSet{\indexset} \rightarrow \R_{\geq 0}$ is a linkage function and  $\alpha_A, \alpha_B, \beta, \gamma \in \R$.
\reftab{tab.Recurrence} lists the coefficient of \refeqn{eq.Recurrence} for some common linkages in \refeqn{eq.Linkage}.
Although  the minimax linkage $\linkage_{M}$ \refeqn{eq.MinimaxLinkage} can not be written in the form of the recurrence formula \cite{bien_tibshirano_JASA2011}, as many other linkage functions above it satisfies 
\begin{align}
\linkage_{M}\prl{\vectbf{x}; A \cup B, C} \geq \min \prl{\big.\linkage_{M}\prl{\vectbf{x}; A,C}, \linkage_{M}\prl{\vectbf{x}; B,C}},
\end{align} 
which is known as the strong reducibility property, defined in the following paragraph. 
%Note that the major advantage of such a recurrence relation is that every agglomerated clusters might be represented by a single observation, generally known as a representative point, and its  dissimilarity with the representative points of other clusters directly defines the linkage criterion between clusters \cite{murtagh_CJ1983}.  

\begin{table}
\caption{Coefficients of the recurrence formula  of  \cite{lance_williams_CJ1967} for some common linkages}
\label{tab.Recurrence}
\begin{tabular}{|c|c|c|c|c|}
\hline 
Linkage & $\alpha_A$ & $\alpha_B$ & $\beta$ & $\gamma$ \\
\hline
\hline
Single  & $0.5$ &  $0.5$ & $0$ & $-0.5$\\
\hline
Complete & $0.5$ & $0.5$ & $0$ & $0.5$ \\
\hline
Average & $\Big.\frac{\card{A}}{\card{A}\sqz{+}\card{B}}$ & $\frac{\card{B}}{\card{A}\sqz{+}\card{B}}$ & $0$ & $0$ \\
\hline
Ward & $\Big.\frac{\card{A} \sqz{+} \card{C}}{\card{A}\sqz{+}\card{B}\sqz{+}\card{C}}$ & $\frac{\card{B} \sqz{+} \card{C}}{\card{A}\sqz{+}\card{B}\sqz{+}\card{C}}$ & $-\frac{\card{C}}{\card{A}\sqz{+}\card{B}\sqz{+}\card{C}}$ & 0 \\
\hline
\end{tabular}
\end{table}

%%%%%%%%%%%%%%%%%%%%%%%%%%%%%%%%%%%%%%%%%%%%%%
%%%%%%%%%%%%%%%%%%%%%%%%%%%%%%%%%%%%%%%%%%%%%%
\subsubsection{Reducibility \& Monotonicity}
\label{sec.ReducibilityMonotonicity}
%%%%%%%%%%%%%%%%%%%%%%%%%%%%%%%%%%%%%%%%%%%%%%
%%%%%%%%%%%%%%%%%%%%%%%%%%%%%%%%%%%%%%%%%%%%%%

%An often preferred property of linkage functions is: 

\begin{definition}[\cite{bruynooghe_CAD1978,murtagh_CJ1983}] \label{def.ReducibleLinkage}
For a fixed finite index set $\indexset$, a linkage function $\linkage:\prl{\R^m}^{\indexset} \times \PowerSet{\indexset} \times \PowerSet{\indexset} \rightarrow \R_{\geq0}$ is said to be \emph{reducible} if for any pairwise disjoint subsets $A,B,C$ of $\indexset$ and set of data points $\vectbf{x} \in \prl{\R^m}^{\indexset}$
\begin{subequations}\label{eq.Reducibility}
\begin{align} \label{eq.ReducibilityCondition}
\linkage\prl{\vectbf{x}; A,B} &\leq \min \prl{\big.\linkage\prl{\vectbf{x}; A, C}, \linkage\prl{\vectbf{x}; B, C}} \end{align}
\end{subequations}
implies
\addtocounter{equation}{-1}
\begin{subequations}
\setcounter{equation}{1}
\begin{align} \label{eq.ReducibilityResult}
 \linkage\prl{\vectbf{x}; A \cup B, C} \geq \min \prl{\big.\linkage\prl{\vectbf{x}; A, C}, \linkage\prl{\vectbf{x}; B, C}}.
\end{align}
\end{subequations}
Further, say $\linkage$ is \emph{strongly reducible}\footnote{Although \cite{gordon_JRSS1987} refers to strong reducibility of linkages as the reducibility property,  by definition, strong reducibility is more restrictive than reducibility of linkages.} if for any pairwise disjoint subsets $A,B,C$ of $\indexset$ and $\vectbf{x} \in \prl{\R^m}^{\indexset}$ it satisfies
\begin{align} \label{eq.StrongReducibility}
 \linkage\prl{\vectbf{x}; A \cup B, C} \geq \min \prl{\big.\linkage\prl{\vectbf{x}; A, C}, \linkage\prl{\vectbf{x}; B, C}}.
\end{align}
%It is convenient to have $\linkageset\prl{\R^d}$ and $\reducibleLinkage\prl{\R^d}$ denote the set of all linkages and  reducible linkages, respectively, for patterns in $\R^d$.
\end{definition}
%
%This is also known as the reducibility property of linkages \cite{murtagh_CJ1983}. 
The well known examples of linkages with the strong reducibility property are single, complete, average and minimax  linkages in \refeqn{eq.Linkage} \cite{murtagh_CJ1983, bien_tibshirano_JASA2011}. 
Even though Ward's linkage is not strongly reducible, it still has the  reducibility property. 

\smallskip

A property of clustering hierarchies (of great importance in the sequel)  consequent upon  the reducibility property of linkages is monotonicity:
% A restatement of monotonicity in clustering hierarchies based on a particular linkage is:   
%
\begin{definition}[\cite{jain_dubes_1988}] \label{def.treeMonotonicity}
A non-degenerate hierarchy $\treetop \in \bintreetopspace_{\indexset}$ associated with a set of data points $\vectbf{x} \in \prl{\R^m}^{\indexset}$  is said to be $\linkage$-\emph{monotone} 
%with respect to a linkage function $\linkage$ or , shortly, $\linkage$-monotone  
if all grandchildren, $I \in \grandchildset{\treetop}$, are more similar to their siblings, $\compLCL{I}{\treetop}$, than are their parents, $ P = \parentCL{I,\treetop}$, i.e. 
\begin{align}
\linkage\prl{\vectbf{x}; I,\compLCL{I}{\treetop}} \leq \linkage\prl{\vectbf{x}; P, \compLCL{P}{\treetop}}. 
\end{align} 
%where $\grandchildset{\treetop}$ \refeqn{eq.grandchildset} is the set of grandchildren of $\treetop$.
\end{definition}

%The well known significance of reducible linkages for the monotonicity of resulting hierarchical clustering  is:
%A major significance of the reducibility property of linkages  is its role on the monotonicity of resulting hierarchical clustering:  
%A well-known fact \cite{jain_dubes_1988} revealing the relation between reducibility and monotonicity is:
% 
\begin{proposition}[\cite{jain_dubes_1988}] \label{prop.treelinkageMonotonicity}
If linkage $\linkage$ is reducible, then a cluster hierarchy $\treetop \in \bintreetopspace_{\indexset}$ in the relation $\HAC_{\linkage}$ (i.e. resulting from procedure of \reftab{tab.HAC} applied to some dataset $\vectbf{x} \in \prl{\R^m}^{\indexset}$) is always $\linkage$-monotone. 
\end{proposition}
\section{Homogeneity}
\label{sec.LocalStructuralOptimality}
%%%%%%%%%%%%%%%%%%%%%%%%%%%%%%%%%%%%%%%%%%%%%%%%
%%%%%%%%%%%%%%%%%%%%%%%%%%%%%%%%%%%%%%%%%%%%%%%%

%We now introduce a new concept of   structural homogeneity for a clustering hierarchy associated with a dataset and a linkage criterion. 
%Then we continue with  its properties and relation with hierarchical agglomerative clustering methods.

We now introduce our new notion of homogeneity and explore its relationships to previously developed structural properties of trees. 
\begin{definition}\label{def.LocalStructuralOptimality}
(Homogeneity) A binary hierarchy $\treetop \in \bintreetopspace_{\indexset}$ associated with a set of data points $\vectbf{x} \in \prl{\R^m}^{\indexset}$ 
is \emph{locally $\linkage$-homogeneous}  at grandchild cluster $I \in \grandchildset{\treetop}$ if the siblings, $I$ and $\compLCL{I}{\treetop}$, are closer to each other than to their parent's sibling, $\compLCL{P}{\treetop} = \compLCL{\parentCL{I,\treetop}}{\treetop}$, 
\begin{align} \label{eq.LocalStructuralOptimality}
\linkage\prl{\vectbf{x}; I, \compLCL{I}{\treetop}} \leq \min\prl{\big. \linkage\prl{\vectbf{x}; I, \compLCL{P}{\treetop}}, \linkage\prl{\vectbf{x}; \compLCL{I}{\treetop}, \compLCL{P}{\treetop}}}.
\end{align}
A tree is \emph{$\linkage$-homogeneous} if it is locally $\linkage$-homogeneous at each grandchild. 
\end{definition}

A useful observation when we  focus attention on reducible linkages is:
%
%Here one might particularly focus attention on the reducible linkages and easily  observe that:   
%
\begin{proposition}\label{prop.LST2Monotonicity}
If a tree,  $\treetop \in \bintreetopspace_{\indexset}$ associated with a set of data points $ \vectbf{x} \in \prl{\R^m}^{\indexset}$, is $\linkage$-homogeneous for a reducible linkage $\linkage$, then it must be $\linkage$-monotone as well. 
%A structurally homogeneous clustering hierarchy $\treetop \in \bintreetopspace_{\indexset}$ associated with a set of data points $\vectbf{x} \in \prl{\R^m}^{\card{\indexset}}$ and a reducible linkage $\linkage$ is always monotone.
\end{proposition}
\begin{proof}
The result directly follows from homogeneity of $\treetop$ and reducibility of $\linkage$. 

For any grandchild cluster $I \in \grandchildset{\treetop}$ and  its parent $P = \parentCL{I,\treetop}$, using \refeqn{eq.Reducibility} and \refeqn{eq.LocalStructuralOptimality},
one can verify the result as
{\small
\begin{align}
\hspace{-1mm}\linkage\prl{\vectbf{x}; \!I, \!\compLCL{I}{\treetop}} \sqz{\leq} \min\prl{\big. \linkage\prl{\vectbf{x}; \! I, \! \compLCL{P}{\treetop}}\!,   \linkage\prl{\vectbf{x}; \!\compLCL{I}{\treetop} \! \!, \! \compLCL{P}{\treetop}}\!} \sqz{\leq} \linkage \prl{\vectbf{x}; \! P, \! \compLCL{P}{\treetop}} \!\!,\!\!\! 
\end{align}
}
where $P = I \cup \compLCL{I}{\treetop}$.
\hfill \qed
\end{proof}

The converse of \refprop{prop.LST2Monotonicity} only holds for single linkage:% among many known linkage functions:
\begin{proposition}\label{prop.SingleLinkageMonotoneLSO}
A clustering hierarchy $\treetop \in \bintreetopspace_{\indexset}$ associated with a set of data points $\vectbf{x} \in \prl{\R^m}^{\indexset}$ is $\linkage_{S}$-monotone for single linkage $\linkage_S$ \refeqn{eq.SingleLinkage}  if and only if it is $\linkage_{S}$-homogeneous as well.
%A clustering hierarchy $\treetop \in \bintreetopspace_{\indexset}$ associated with a set of data points  $\vectbf{x} \in \prl{\R^m}^{\card{\indexset}}$ and single linkage $\linkage_{S}$ \refeqn{eq.SingleLinkage} is monotone if and only if it is  structurally homogeneous.
\end{proposition}
\begin{proof}
The sufficiency of  $\linkage_{S}$-homogeneity of a clustering tree for its $\linkage_{S}$-monotonicity directly follows from \refprop{prop.LST2Monotonicity}.

The other way of implication is  evident from definitions of monotonicity (\refdef{def.treeMonotonicity}) and  single linkage $\linkage_{S}$ \refeqn{eq.SingleLinkage}, i.e. for any $I \in \grandchildset{\treetop}$ and $P = \parentCL{I,\treetop}$,
{\small
\begin{align}
\linkage_S \prl{\vectbf{x}; I, \compLCL{I}{\treetop}} &\leq \linkage_S\prl{\vectbf{x};P, \compLCL{P}{\treetop}}, \\
&= \min \prl{ \big.\linkage_S\prl{\vectbf{x}; I, \compLCL{P}{\treetop}}, \linkage_S\prl{\vectbf{x}; \compLCL{I}{\treetop}, \compLCL{P}{\treetop}}}, 
\end{align}
}    
where $P = I \cup \compLCL{I}{\treetop}$. \hfill \qed
\end{proof}

A major significance of homogeneity is that it is  a common characteristic feature of any clustering hierarchy resulting from   agglomerative clustering using any strong reducible linkage: 
\begin{proposition}\label{prop.HAC2LSO}
If linkage $\linkage$ is  strongly reducible then
any  non-degenerate hierarchy $\treetop \in \bintreetopspace_{\indexset}$ in the relation $\HAC_\linkage$ (i.e. resulting from the procedure of \reftab{tab.HAC} applied to some dataset $\vectbf{x} \in \prl{\R^m}^{\indexset}$ ) is $\linkage$-homogeneous. 
%If a non-degenerate hierarchy $\treetop \in \bintreetopspace_{\indexset}$ results from  hierarchical agglomerative $\linkage$-clustering of a set of data points $\vectbf{x} \in \prl{\R^m}^{\card{\indexset}}$ using a strongly reducible linkage $\linkage$ then $\treetop$  homogeneous with respect to $\linkage$.
%A non-degenerate hierarchy $\treetop \in \bintreetopspace_{\indexset}$ resulting from  hierarchical agglomerative clustering of a set of data points $\vectbf{x} \in \prl{\R^m}^{\card{\indexset}}$ based on a strongly reducible linkage $\linkage$ is always structurally homogeneous.  
\end{proposition}
\begin{proof}
Let $\prl{\mathcal{J}_{k}}_{0 \leq k \leq \card{J} - 1}$ be a sequence of nested partitions  of $\indexset$, defining $\treetop$ as in \refeqn{eq.Partition2Cluster},  resulting from  agglomerative $\linkage$-clustering of  $\vectbf{x}$. 
Further, for $ 0 \leq k \leq \card{\indexset}-2 \;$ let  $\prl{A^*_k,B^*_k} $  be a pair of clusters of $\mathcal{\indexset}_k$ in \refeqn{eq.LinkageOptimal}  with the minimum linkage value.
%Also, recall from \refeqn{eq.grandchild} that $\grandchildset{\treetop} = \bigcup_{k = 0}^{\card{\indexset} - 3} \crl{A^*_k, B^*_k}$.

For $0 \leq k \leq \card{\indexset} -3 $ and any (grandchild) cluster $C_k \in \mathcal{\indexset}_{k} \setminus \crl{A^*_k, B^*_k}$, from \refeqn{eq.LinkageOptimal},  we have
\begin{align}
\linkage\prl{\vectbf{x}; A^*_k, B^*_k} \leq \linkage\prl{\vectbf{x}; A_k^*, C_k}, \label{eq.Ak_monotonicity}\\
\linkage\prl{\vectbf{x}; A^*_k, B^*_k} \leq \linkage\prl{\vectbf{x}; B_k^*, C_k}.
\end{align}  

Now, observe that the parent's sibling $\compLCL{\prl{A^*_k \cup B^*_k}}{\treetop}$ of $A^*_k$ and $B^*_k$
% the local complement (sibling) of $A^*_k \cup B^*_k = \parentCL{A^*_k,\treetop} = \parentCL{B^*_k,\treetop}$ 
can be written as the union of elements of a subset $\mathcal{D}$ of $\mathcal{\indexset}_k \setminus \prl{A^*_k, B^*_k}$,
\begin{align}
\compLCL{\prl{A^*_k \cup B^*_k}}{\treetop} = \bigcup_{D \in \mathcal{D} } D.
\end{align} 
That is to say, the elements of $\mathcal{D}$ are merged in a way described by the sequence of nested partitions $\prl{\mathcal{J}_{k}}_{0 \leq k \leq \card{J} - 1}$ of $\indexset$ such that their union finally yields $\compLCL{\prl{A^*_k \cup B^*_k}}{\treetop}$.

Hence, using strong reducibility of $\linkage$ and \refeqn{eq.Ak_monotonicity}, one can verify that
\begin{align}
\linkage\prl{\vectbf{x}; A^*_k, \compLCL{\prl{A^*_k\cup B^*_k}}{\treetop}} &= \linkage\prl{\vectbf{x};A^*_k, \bigcup_{D \in \mathcal{D}} D}, \\
& \geq \min_{D \in \mathcal{D}} \linkage\prl{\vectbf{x}; A^*_k, D}, \\
&\geq \linkage\prl{\vectbf{x}; A^*_k, B^*_k},
\end{align}
which, by symmetry, also holds for $B^*_k$,
\begin{align}
\linkage\prl{\vectbf{x}; B^*_k, \compLCL{\prl{A^*_k\cup B^*_k}}{\treetop}} \geq \linkage\prl{\vectbf{x}; A^*_k, B^*_k}.
\end{align}

Thus, since   $\grandchildset{\treetop} = \bigcup_{k = 0}^{\card{\indexset} - 3} \crl{A^*_k, B^*_k}$ \refeqn{eq.Partition2GrandChild},  the result follows.
\hfill \qed
\end{proof}

%\ToDo{Coming soon: A theorem relating local structural optimality to single linkage clustering.}

In particular, a critical observation for  single linkage is: % relating the local structural optimality of a clustering tree to a globally optimal result of hierarchical clustering: 
%for the single linkage $\linkage_{S}$ \refeqn{eq.SingleLinkage} is as follows:

\begin{theorem}\label{thm.SingleLinkage}
A   non-degenerate hierarchy $\treetop \in \bintreetopspace_{\indexset}$ is in the relation $\HAC_{\linkage_S}$  (i.e. results from the procedure of \reftab{tab.HAC} applied to some dataset $\vectbf{x} \in \prl{\R^m}^{\indexset}$  using  $\linkage_{S}$ \refeqn{eq.SingleLinkage} as the linkage) if and only if
 $\treetop$ is $\linkage_{S}$-homogeneous (or, equivalently, $\linkage_{S}$-monotone).
%A non-degenerate hierarchy $\treetop\in\bintreetopspace_{\indexset}$  results from  hierarchical agglomerative $\linkage_{S}$-clustering of a set of data points $\vectbf{x} \in \prl{\R^m}^{\card{\indexset}}$ using single linkage, $\linkage_{S}$ \refeqn{eq.SingleLinkage}, if and only if $\treetop$ is homogeneous (equivalently, monotone) with respect to $\linkage_{S}$.
%A clustering hierarchy $\treetop \in \bintreetopspace_{\indexset}$ associated with a set of data points $\vectbf{x} \in \prl{\R^m}^{\card{\indexset}}$ and  single linkage $\linkage_{S}$ \refeqn{eq.SingleLinkage} is  structurally homogeneous (or, equivalently, monotone --- \reflem{lem.SingleLinkageMonotoneLSO})
%, if and only if it is monotone, 
%if and only if it is a result\;\footnote{Note that single linkage clustering of a dataset yields a unique clustering hierarchy if a minimum spanning tree of the dataset is unique, which is the case if all pairwise distances between data points are distinct  \cite{gower_ross_jrss1969, gallager_humblet_spira_TOPLAS1983}.} of agglomerative single linkage clustering of $\vectbf{x}$.
\end{theorem}
\begin{proof}
%Recall from \reflem{lem.SingleLinkageMonotoneLSO} that  for  single linkage $\linkage_{S}\;$  structural homogeneity and monotonicity of a clustering hierarchy imply each other.  
The sufficiency, of being a single linkage clustering hierarchy, for   homogeneity is evident from \refprop{prop.HAC2LSO}.%, \refprop{prop.SingleLinkageMonotoneLSO} and \refprop{prop.treelinkageMonotonicity}.

To see the necessity of homogeneity, we will first prove that if $\treetop$ is $\linkage_{S}$-homogeneous, then for any $I \in \grandchildset{\treetop}$ and nonempty subset $Q \subseteq \indexset \setminus \parentCL{I,\treetop}$ the following holds
\begin{align}
\linkage_S\prl{\vectbf{x}; I, \compLCL{I}{\treetop}} \leq \linkage_{S}\prl{\vectbf{x}; I, Q}. \label{eq.SingleLinkageMonotone}
\end{align}
Observe that \refeqn{eq.SingleLinkageMonotone} states that the cost of merging any one of $I$ and $\compLCL{I}{\treetop}$ with another cluster $Q \subseteq \indexset \setminus \parentCL{I,\treetop}$ is greater and equal to the cost of merging $I$ and $\compLCL{I}{\treetop}$.
Then, by induction, we conclude that $\treetop$ is a possible outcome of agglomerative single linkage clustering of $\vectbf{x}$.

\smallskip

Let $\ancestorCL{I,\treetop}$ denote the set of ancestors of cluster $I \in \cluster{\treetop}$ of $\treetop$, except the root  $\indexset$,
\begin{align}
\ancestorCL{I,\treetop} \ldf \crl{A \in \cluster{\treetop} \setminus \crl{\indexset} \big | I \subsetneq A}.
\end{align} 

Using the definition of $\linkage_{S}$ \refeqn{eq.SingleLinkage} and monotonicity of $\treetop$, one can  verify that 
for any grandchild $I\in \grandchildset{\treetop}$ and  its ancestor $ A \in  \ancestorCL{I,\treetop}$
\begin{align}
\linkage_S\prl{\vectbf{x}; I, \compLCL{I}{\treetop}} \leq \linkage_S\prl{\vectbf{x}; A, \compLCL{A}{\treetop}} \leq \linkage_S \prl{\vectbf{x}; I, \compLCL{A}{\treetop}}. \label{eq.MonBound}
\end{align} 
Now observe that the global complement of $\parentCL{I,\treetop}$ can be written as
\begin{align}
\indexset \setminus \parentCL{I,\treetop} = \bigcup_{A \in \ancestorCL{I,\treetop}} \compLCL{A}{\treetop}. \label{eq.CompUnion}
\end{align} 
As a result, combining \refeqn{eq.MonBound} and \refeqn{eq.CompUnion} yields
\begin{align}
\linkage_S\prl{\vectbf{x}; I, \indexset \setminus \parentCL{I,\treetop}} &= \linkage_S\prl{\vectbf{x}, I, \bigcup_{A \in \ancestorCL{I,\treetop}} \compLCL{A}{\treetop}}\\
&= \min_{A \in \ancestorCL{I,\treetop}} \linkage_S\prl{\vectbf{x}; I; \compLCL{A}{\treetop}} \\
& \geq \linkage_{S}\prl{\vectbf{x}; I, \compLCL{I}{\treetop}},
\end{align}
from which one can conclude \refeqn{eq.SingleLinkageMonotone} for  single linkage $\linkage_{S}$.

Finally, using a proof by induction, the result of the theorem can be shown as follows:
\begin{itemize}
\item (Base Case) if $I, \compLCL{I}{\treetop} \in \cluster{\treetop}$ are singleton clusters, then, since they satisfy \refeqn{eq.SingleLinkageMonotone}, they can be merged at an appropriate step of the agglomerative clustering when the minimum cluster distance is equal to $\linkage_{S}\prl{\vectbf{x}; I, \compLCL{I}{\treetop}}$. 
Note that, due to \refeqn{eq.SingleLinkageMonotone},  neither $I$ nor $\compLCL{I}{\treetop}$ can be merged with any other cluster $Q \subseteq \indexset \setminus \parentCL{I,\treetop}$ at a lower linkage value than $\linkage_{S}\prl{\vectbf{x}; I, \compLCL{I}{\treetop}}$.

\item (Induction) Otherwise, suppose that  $I$ and $\compLCL{I}{\treetop}$ are already constructed since their children also satisfy \refeqn{eq.SingleLinkageMonotone} and, by monotonicity, children $\crl{I_L, I_R} =\childCL{I,\treetop}$ of $I$   satisfies  $\linkage_{S}\prl{\vectbf{x}; I_L, I_R} \leq \linkage_S\prl{\vectbf{x}; I, \compLCL{I}{\treetop}}$ as do children of $\compLCL{I}{\treetop}$.
Thus,  since clusters $I$ and $\compLCL{I}{\treetop}$ satisfy \refeqn{eq.SingleLinkageMonotone}, they   can be directly aggregated when the merging cost, i.e. the value of minimum cluster distance in \refeqn{eq.LinkageOptimal},  reaches $\linkage_S\prl{\vectbf{x}; I, \compLCL{I}{\treetop}}$.    
\end{itemize}  

\hfill \qed
\end{proof}

%%%%%%%%%%%%%%%%%%%%%%%%%%%%%%%%%%%%%%%%%%%%%
%%%%%%%%%%%%%%%%%%%%%%%%%%%%%%%%%%%%%%%%%%%%%
\section{Anytime Hierarchical Clustering}
\label{sec.OnlineHierarchicalClustering}
%%%%%%%%%%%%%%%%%%%%%%%%%%%%%%%%%%%%%%%%%%%%%
%%%%%%%%%%%%%%%%%%%%%%%%%%%%%%%%%%%%%%%%%%%%%

%In this section we introduce a simple online anytime hierarchical clustering method, and then continue with its termination and complexity analysis  for a family of linkages.  

Given a choice of linkage, $\linkage: \prl{\R^m}^{\indexset} \times \PowerSet{\indexset} \times \PowerSet{\indexset} \rightarrow \R_{\geq 0}$, \reftab{tab.OHC} presents the formal specification of our central contribution, the associated Anytime Hierarchical $\linkage$-Clustering method. 
Once again, this method defines a new relation from datasets to hierarchies, $\AHC_{\linkage} \subset \prl{\R^m}^{\indexset} \times \bintreetopspace_{\indexset}$ that is generally not a function but rather a multi-function (i.e. all datasets yield some hierarchy, but not necessarily a unique one).

\begin{table}[t]
\caption{Anytime Hierarchical $\linkage$-Clustering}
\label{tab.OHC}
\begin{tabular}{|p{0.46\textwidth}|}
\hline
\\[-2.5mm]
For any  given clustering hierarchy $\treetop \in \bintreetopspace_{\indexset}$ associated with a set of data points $\vectbf{x} \in \prl{\R^{m}}^{\indexset}$, and  linkage function $\linkage: \prl{\R^{m}}^{\indexset} \times \PowerSet{\indexset} \times \PowerSet{\indexset} \rightarrow \R_{\geq 0}$,  
 
\begin{enumerate}
\item If $\treetop$ is  $\linkage$-homogeneous, then terminate and return $\treetop$.% as the clustering tree of $\vectbf{x}$ based on linkage $\linkage$. 

\item Otherwise, 

\begin{enumerate}

\item Find a grandchild cluster $I \in \grandchildset{\treetop}$ at which $\treetop$ violates  local homogeneity, i.e. 
{\small
\begin{align}
\hspace{-2mm} \linkage\prl{\vectbf{x}; I, \compLCL{I}{\treetop}} \sqz{>} \min \prl{ \big. \linkage\prl{\vectbf{x};\! I, \!\compLCL{P}{\treetop}} \!\!, \linkage\prl{\vectbf{x}; \! \compLCL{I}{\treetop}\!,\! \compLCL{P}{\treetop}} \!} \!, \!\!
\end{align} 
}
where $P = \parentCL{I,\treetop}$.
\item Then perform an NNI restructuring  on $\treetop$ at  grandchild $G^*  \in \childCL{P,\treetop}$ with the maximum dissimilarity to $\compLCL{P}{\treetop}$, i.e. swap $G^*$ with $\compLCL{P}{\treetop}$,
\begin{subequations}
\begin{align}
G^* = \argmax_{G \in \childCL{P,\treetop} } \linkage\prl{\vectbf{x}; G, \compLCL{P}{\treetop}},\\
\treetop \leftarrow \NNI \prl{\treetop, G^*}, \hspace{15mm}
\end{align}
\end{subequations}
and go to Step 1.
\end{enumerate}
\vspace{-5mm}
\end{enumerate}  
\\
\hline
\end{tabular}
\end{table}

Because the  procedure defining  $\AHC_{\linkage}$ in \reftab{tab.OHC} does not entail any obvious gradient-like  greedy step as do many previously proposed iterative clustering methods, demonstrating that it terminates requires some analysis that we now present.

\subsection{Proof of Convergence}
\label{sec.Stability}
%%%%%%%%%%%%%%%%%%%%%%%%%%%%%%%%%%%%%%
%%%%%%%%%%%%%%%%%%%%%%%%%%%%%%%%%%%%%%

For any non-degenerate hierarchy $\treetop \in \bintreetopspace_{\indexset}$ associated with a set of data points  $\vectbf{x} \in \prl{\R^m}^{\indexset}$ and a linkage function $\linkage$, we consider  the sum of linkage values as an objective function to assess the quality of clustering, 
\begin{align} \label{eq.totallinkagecost}
H_{\vectbf{x}, \linkage}\prl{\treetop} \ldf \frac{1}{2}\sum_{I \in \cluster{\treetop}} \linkage\prl{\vectbf{x}; I, \complementLCL{I}}.
\end{align}
Intuitively, one might expect that hierarchical agglomerative clustering methods yield clustering hierarchies minimizing \refeqn{eq.totallinkagecost}. 
However, they are generally known to be step-wise optimal greedy methods \cite{gordon_JRSS1987} with an exception that single linkage clustering always returns a globally optimal clustering tree in the sense of \refeqn{eq.totallinkagecost} due to its close relation with a minimum spanning tree of the data set \cite{gower_ross_jrss1969}.
In contrast, for example, as witness to the general sub-optimality of agglomerative clustering relative to \refeqn{eq.totallinkagecost}, for Ward's linkage  $\linkage_W$ \refeqn{eq.WardLinkage} $H_{\vectbf{x}, \linkage_{W}}$  is constant and equal to the sum of squared error of $\vectbf{x}$ (see \refapp{app.WardLinkage}), i.e.  for any $\treetop \in \bintreetopspace_{\indexset}$
\begin{align}
H_{\vectbf{x}, \linkage_{W}}\prl{\treetop} = \sum_{i \in \indexset} \norm{\vect{x}_i - \ctrd{\vectbf{x}|\indexset}}_2^2,
\end{align}
where $\ctrd{\vectbf{x}| \indexset}$ \refeqn{eq.centroid} denotes the centroid of $\vectbf{x}|\indexset$. 

\smallskip

Let $\prl{\treeA,\treeB}$ be a pair of  NNI-adjacent (\refdef{def:NNIMove}) hierarchies in $\bintreetopspace_{\indexset}$ and $\prl{A,B,C}$ be the NNI-triplet (\reflem{lem.NNITriplet}) of common clusters of $\treeA$ and $\treeB$.
Recall that $A\cup B \in \cluster{\treeA}\setminus \cluster{\treeB}$ and $B \cup C \in \cluster{\treeB}\setminus \cluster{\treeA}$ are only unshared clusters of $\treeA$ and $\treeB$, respectively.
Hence, one can write the change in the objective function $H_{\vectbf{x}, \linkage}$ \refeqn{eq.totallinkagecost} after the NNI transition from $\treeA$ to $\treeB$ as
{
\begin{align} \label{eq.ChangeTotalLinkageCost}
 H_{\vectbf{x}, \linkage} \prl{\treeB} \sqz{-} H_{\vectbf{x}, \linkage} \prl{\treeA}& \sqz{=}    \linkage\prl{\vectbf{x}; B,C} \sqz{+} \linkage\prl{\vectbf{x}; B\sqz{\cup} C, A} \nonumber \\
 & \hspace{10mm} \sqz{-} \linkage\prl{\vectbf{x}; A,B} \sqz{-} \linkage\prl{\vectbf{x}; A\sqz{\cup} B, C}.
\end{align}
}
Here we find it useful to  define a new class of linkages: %based on the change  in the objective function $ H_{\vectbf{x}, \linkage}$ \refeqn{eq.ChangeTotalLinkageCost} and the form of our online hierarchical clustering method in \refsec{sec.OnlineHierarchicalClustering}:
\begin{definition} \label{def.Contraction}
A linkage  $\linkage : \prl{\R^m}^{\indexset} \times \PowerSet{\indexset} \times \PowerSet{\indexset} \rightarrow \R_{\geq 0} $ is \emph{NNI-reducible} if for any set of data points $\vectbf{x} \in \prl{\R^m}^{\indexset}$ and pairwise disjoint subsets $A,B,C$ of $\indexset$ 
{
\begin{subequations} \label{eq.NNI-Reducibility}
\begin{align}
\linkage\prl{\vectbf{x};B,C} \leq \min \prl{ \big. \linkage\prl{\vectbf{x}; A,B}, \linkage\prl{\vectbf{x}; A,C}}.
%\linkage\prl{\vectbf{x}; A,B} > \linkage\prl{\vectbf{x};A,C}, \quad \text{and} \quad
%\linkage\prl{\vectbf{x}; B, C} \geq \linkage\prl{\vectbf{x};A,C}, 
\end{align}
\end{subequations}
}
implies
\addtocounter{equation}{-1}
{\small
\begin{subequations}
\setcounter{equation}{1}
\begin{align}
\hspace{-0mm}\min\prl{ \!\!\!\!
\begin{array}{l}
\linkage\prl{\vectbf{x}; \!A,\!B} \sqz{+} \linkage\prl{\vectbf{x}; \!A \sqz{\cup} B, \!C}\!, \\ 
\linkage\prl{\vectbf{x}; \! A, \! C} \sqz{+} \linkage\prl{\vectbf{x}; \! A \sqz{\cup} C, \! B} 
\end{array} \!\!\!\!\!} \sqz{\geq} \linkage\prl{\vectbf{x}; \! B, \! C} \sqz{+} \linkage\prl{\vectbf{x}; \! B\sqz{\cup} C, \! A}\!.\!\!\!
\end{align}
\end{subequations}
}
\end{definition}
Using \refeqn{eq.Linkage}, \refeqn{eq.Recurrence} and \reftab{tab.Recurrence} one can verify that single, complete, minimax and Ward's linkages  are examples of  NNI-reducible linkages.
Note that  a reducible linkage is not necessarily  NNI-reducible; for instance, average linkage $\linkage_{A}$ \refeqn{eq.AverageLinkage}. 
%Likewise, a contraction linkage is not always a reducible linkage; for instance, the variance linkage $\linkage_{V}$,
% 
%\begin{align}\label{eq.VarianceLinkage}
%\linkage_{V}\prl{\vectbf{x}; A,B} \ldf  \frac{1}{\card{A \cup B} } %\sum_{a \in A \cup B} \norm{\vect{x}_a - \ctrd{\vectbf{x}|A\cup B}}^2,
%\end{align}
%
%for any $\vectbf{x} \in \prl{\R^d}^{\card{\indexset}}$ and disjoint subsets $A,B \subset \indexset$, where $\ctrd{\vectbf{x}|I}$ \refeqn{eq.centroid} is the centroid of $\vectbf{x}|I$ for any $I \subseteq \indexset$. 

\smallskip

We now proceed to investigate the termination  of anytime hierarchical clustering for NNI-reducible linkages:
\begin{lemma} \label{lem.linkageOptimality}
For any set of data points $\vectbf{x} \in \prl{\R^m}^{\indexset}$ and  NNI-reducible linkage $\linkage$, 
the value of objective function $H_{\vectbf{x},\linkage}$ \refeqn{eq.totallinkagecost} is nonincreasing at each iteration of anytime hierarchical clustering  in \reftab{tab.OHC} from any  initial hierarchy $\treeA \in \bintreetopspace_{\indexset}$ towards $\treeB \in \bintreetopspace_{\indexset}$,
\begin{align}
H_{\vectbf{x}, \linkage}\prl{\treeB}  - H_{\vectbf{x}, \linkage}\prl{\treeA} \leq 0.
\end{align}
\end{lemma}
\begin{proof}
If $\treeA$ is  $\linkage$-homogeneous, then $\treeB = \treeA$, and so the result directly follows.

Otherwise, let $\prl{A,B,C}$ be the NNI-triplet (\reflem{lem.NNITriplet}) associated with  $\prl{\treeA, \treeB}$. 
Recall that  $A \cup B \in \cluster{\treeA} \setminus \cluster{\treeB}$ and $B \cup C \in \cluster{\treeB} \setminus \cluster{\treeA}$. 
To put it another way,  anytime hierarchical clustering performs an NNI move on $\treeA$ at grandchild $A \in \grandchildset{\treetop}$ towards $\treeB$, and so 
\begin{align}
\linkage\prl{\vectbf{x}; A, B} &> \linkage\prl{\vectbf{x}; B,C}, \\
\linkage\prl{\vectbf{x}; A,C} &\geq \linkage\prl{\vectbf{x}; B,C}. 
\end{align}  
Therefore, since $\linkage$ is NNI-reducible (\refdef{def.Contraction}), the change in the objective function $H_{\vectbf{x}, \linkage}$ \refeqn{eq.ChangeTotalLinkageCost} is nonnegative,
\begin{align}
 H_{\vectbf{x}, \linkage}\prl{\treeB} -H_{\vectbf{x}, \linkage}\prl{\treeA} 
 %& =     \underbrace{ \linkage\prl{\vectbf{x};A,C} + \linkage\prl{\vectbf{x}; A\sqcup C, B} - \linkage\prl{\vectbf{x};A,B} - \linkage\prl{\vectbf{x}; A\sqcup B, C},}_{\leq 0, \text{ by \refdef {def.Contraction}}} \\
&\leq 0,
\end{align}
which completes the proof.
\hfill \qed
\end{proof}
%

%Accordingly, one can characterize the termination properties of 

\begin{theorem} \label{thm.Convergence}
If $\linkage$ is an NNI-reducible linkage, then iterated application of the Anytime Hierarchical $\linkage$-Clustering procedure  of \reftab{tab.OHC} initiated from any hierarchy in $\bintreetopspace_{\indexset}$ for a fixed set of data points $\vectbf{x} \in \prl{\R^m}^{\indexset}$ must terminate in finite time at a tree in  $\bintreetopspace_{\indexset}$, that is $\linkage$-homogeneous. 
%Online hierarchical clustering (\refsec{sec.OnlineHierarchicalClustering}) of  a set of data points $\vectbf{x} \in \prl{\R^m}^{\card{\indexset}}$ based on an NNI-reducible linkage $\linkage$, starting from any initial clustering tree in $\bintreetopspace_{\indexset}$, terminates at a structurally homogeneous clustering hierarchy in finite time.
\end{theorem}
\begin{proof}
For a fixed finite index set $\indexset$, the number of non-degenerate hierarchies  in $\bintreetopspace_{\indexset}$ \refeqn{eq.numBinTree}  is finite. 
Hence, for the proof of theorem, we shall show that  the anytime  clustering procedure in \reftab{tab.OHC} can not yield any cycle in $\bintreetopspace_{\indexset}$.

Let $\treetop^k \in \bintreetopspace_{\indexset}$ denote a clustering hierarchy visited at $k$-th iteration of anytime clustering method, where $k \geq 0$.
Since $\treetop^k$ and $\treetop^{k+1}$ are NNI-adjacent, let $\prl{A^k,B^k,C^k}$ be the associated NNI-triplet (\reflem{lem.NNITriplet}) of the pair $\prl{\treetop^k, \treetop^{k+1}}$ satisfying   $A^k \cup B^k \in \cluster{\treetop^k} \setminus \cluster{\treetop^{k+1}}$ and $B^k \cup C^k \in \cluster{\treetop^{k+1}} \setminus \cluster{\treetop^k}$.
Further, recall from \reflem{lem.linkageOptimality} that for any NNI-reducible linkage $\linkage$,  $H_{\vectbf{x}; \linkage}\prl{ \treetop^{k+1}} - H_{\vectbf{x}; \linkage}\prl{ \treetop^{k}}\leq 0$. 

If  $H_{\vectbf{x}; \linkage}\prl{ \treetop^{k+1}} - H_{\vectbf{x}; \linkage}\prl{ \treetop^{k}} < 0$, it is clear that anytime clustering method never revisits any previously visited clustering hierarchy. 

Otherwise, $H_{\vectbf{x}; \linkage}\prl{ \treetop^{k+1}} = H_{\vectbf{x}; \linkage}\prl{ \treetop^{k}}$,  we have
{\small
\begin{align}
\hspace{-2mm}\linkage \!\prl{\!\vectbf{x}; \!A^k\!, \! B^k \!} \sqz{+} 
\linkage\!\prl{\! \vectbf{x}; \!A^k\sqz{\cup} B^k\!,\! C^k \!} 
&\sqz{=} 
\linkage\!\prl{\!\vectbf{x}; \! B^k \! , \! C^k \!} \sqz{+} 
\linkage\!\prl{\! \vectbf{x}; \! B^k \sqz{\cup} C^k\!, \! A^k \!} \!\!, \! \! \!\!\\
\linkage\prl{\vectbf{x}; \!A^k \!, \!B^k} &\sqz{>} \linkage\prl{\vectbf{x};\!B^k,\! C^k}, 
\end{align}
}
where the later is due to the anytime clustering rule in \reftab{tab.OHC}. 
Hence, the construction cost of (grand)parent $P^k = A^k \cup B^k \cup C^k$ increases after the NNI move,
\begin{align} \label{eq.GrandParentMonotone}
\linkage\prl{\vectbf{x}; A^k \cup B^k, C^k} < \linkage\prl{\vectbf{x}; B^k \cup C^k, A^k}.
\end{align}   

Now, let $\depth{\treetop}\prl{I}$ denote the level of cluster $I \in \cluster{\treetop}$ of $\treetop \in \bintreetopspace_{\indexset}$ which is equal to the number of ancestors of $I$ in $\treetop$,
\begin{align}
\depth{\treetop}\prl{I} \ldf \card{\crl{A \in \cluster{\treetop} \big| I \subseteq A}},
\end{align}  
and define $L\prl{\treetop}$ to be an ordered $\prl{\card{\indexset}-1}$-tuple of sum of linkages of $\treetop$ at each level,  
\begin{align}
L \prl{\treetop} &\ldf \prl{\big. f_{\treetop}\prl{t}}_{1\leq t \leq \card{\indexset} -1 }, \nonumber \\
 &= \prl{\big. f_{\treetop}\prl{1}, f_{\treetop}\prl{2}, \ldots, f_{\treetop}\prl{\card{\indexset}-1}},
\end{align}
where a binary hierarchy over leaf set $\indexset$ might have at most $\card{\indexset} -1$ levels, and 
\begin{align}
f_{\treetop}\prl{t} \ldf \frac{1}{2} \sum_{\substack{I \in \cluster{\treetop}\\ \depth{\treetop}\prl{I} = t}} \linkage\prl{\vectbf{x}; I, \compLCL{I}{\treetop}}.
\end{align}
Note that if there is no cluster at level $t$ of $\treetop$, then we set $f_{\treetop}\prl{t} = 0$.

Have $(\card{\indexset}-1)$-tuples of real numbers ordered lexicographically  according to the standard order of reals.
Then, since NNI transition from $\treetop^k$ to $\treetop^{k+1}$ might only change linkages between clusters below (grand)parent cluster $P^k = A^k \cup B^k \cup C^k$, using \refeqn{eq.GrandParentMonotone}, one can conclude  that
\begin{align}
L \prl{\treetop^k} < L \prl{\treetop^{k+1}}.
\end{align}
Thus, it is also impossible to visit the same clustering hierarchy at the same level of objective function $H_{\vectbf{x}, \linkage}$, which completes the proof. \hfill \qed

%For any $\treetop \in \bintreetopspace_{\indexset}$, let $S\prl{\treetop}$ be a sorted array  in ascending order  defined to be
%%
%\begin{align}
%S \prl{\treetop} \ldf \prl{\hierconst^{\depth{\treetop}\prl{I}} \linkage\prl{\vectbf{x}; I_L, I_R}}_{\substack{I \in \cluster{\treetop} \\ \crl{I_L, I_R} = \childCL{I,\treetop}}},
%\end{align}
%%
%where $\hierconst \in \prl{0, 1}$. Note that $\card{S\prl{\treetop}} = \card{\cluster{\treetop}} = 2 \card{\indexset} - 1$ and $S\prl{\treetop}_i \leq S\prl{\treetop}_{i+1}$ for all $1 \leq i< \card{\indexset} - 1 $, where $S\prl{\treetop}_i$ denotes the $i$th element of $S\prl{\treetop}$.
%
%
%Now, let  $\treetop_0 \in \bintreetopspace_{\indexset} \setminus \mathcal{G}_{\vectbf{x}, \linkage}$ and  $\treetop_k \in \bintreetopspace_{\indexset}$ be the initial  and  current state of the dynamical system at time $k > 0$. Let $i^*_k$ denotes the index of first difference between  
%
%Let $\treetop_k \in \bintreetopspace_{\indexset}$  
%
% \hfill \qed
\end{proof}

Even though average linkage $\linkage_A$ \refeqn{eq.AverageLinkage} is not NNI-reducible, as shown in \refapp{app.AverageLinkageTermination}, anytime hierarchical clustering based on average linkage still has the finite time termination property.

%%%%%%%%%%%%%%%%%%%%%%%%%%%%%%%%%%%%%%%%%%%
%%%%%%%%%%%%%%%%%%%%%%%%%%%%%%%%%%%%%%%%%%%
\subsection{A Brief Discussion of Computational \\ Properties}
\label{sec.ComplexityAnalysis}
%%%%%%%%%%%%%%%%%%%%%%%%%%%%%%%%%%%%%%%%%%%
%%%%%%%%%%%%%%%%%%%%%%%%%%%%%%%%%%%%%%%%%%%

Complexity analysis of any recursive algorithm will necessarily engage  two logically independent questions:  (i) how many iterations are required to convergence; and (ii) what computational cost is incurred by application of the recursive function at each step along the way? 
Accordingly, in this section we address this pair of question in the context of the anytime hierarchical clustering algorithm of \reftab{tab.OHC}. 
Specifically we : (i) discuss (but defer to a subsequent paper a complete treatment of) the problem of determining bounds on the number of iterations of anytime clustering; and (ii) present explicit bounds on the computational cost of checking whether a cluster hierarchy violates local homogeneity at a given cluster (node) of tree or not.
Prior work on discriminative comparison of non-degenerate hierarchies \cite{ArslanEtAl_NNITechReport2013} and the results of experimental
evaluation in \refsec{sec.ExperimentalEvaluation} hint at a bound on the number of iterations (i) that is $\bigO{n^2}$ with the dataset cardinality, $n$. 
% evaluation in \refsec{sec.ExperimentalEvaluation} hint a polynomial bound on the number of iterations with the number of dataset size, $n$; likely to be $\bigO{n^2}$.
We leave a comprehensive detailed study of algorithmic complexity of anytime hierarchical clustering  to a future discussion of specific implementations.     
However, we still find it useful to give a brief idea of  the computational cost incurred by the determination of tree homogeneity with respect to a number of commonly used linkages.

A straightforward implementation to check (ii) local homogeneity  of a clustering hierarchy at any cluster with respect to any  linkage function in \refeqn{eq.Linkage}   generally has  time complexity of $\bigO{n^2}$  with the dataset size, $n$, with an exception that local homogeneity of a clustering tree relative to Ward's linkage can be computed in  linear, $\bigO{n}$, time.

Alternatively, following the CF(Clustering Feature) tree of BIRCH \cite{zhang_ramakrishnan_livny_SIGMOD1996}, a simple tree data structure can be used to store  sufficient statistics, such as cluster sizes, means and variances, of a clustering hierarchy associated with a dataset. 
Such a data structure can be constructed  in linear time, with the dataset size, using a post-order traversal of a clustering tree and the following recursion of cluster sizes, means and variances.
For any $\vectbf{x} \in \prl{\R^{m}}^{\indexset}$ and disjoint subsets $A,B \subseteq \indexset$ of a finite index set $\indexset$, the sufficient statistics of  $\vectbf{x}|A \cup B$ can be written in terms of the sufficient statistics of $\vectbf{x}|A$ and $\vectbf{x}|B$ as follows:\footnote{A slightly different form of  \refeqn{eq.SufficientStatisticRecursion}  is known as the additivity theorem of CF trees of \cite{zhang_ramakrishnan_livny_SIGMOD1996}.}
{\small
\begin{subequations} \label{eq.SufficientStatisticRecursion}
\begin{align} 
\card{A \cup B} &\sqz{=} \card{A} + \card{B},\\
\ctrd{\vectbf{x}|A\sqz{\cup} B} &\sqz{=} \scalebox{1}{$\frac{\card{A}}{\card{A}+\card{B}}$}\ctrd{\vectbf{x}|A} \sqz{+} \scalebox{1}{$\frac{\card{B}}{\card{A}+\card{B}}$}\ctrd{\vectbf{x}|B}\!, \\
\var{\vectbf{x}|A\sqz{\cup} B} &\sqz{=} \scalebox{1}{$\frac{\card{A}}{\card{A}\sqz{+}\card{B}}$} \var{\vectbf{x}|A} \sqz{+} \scalebox{1}{$\frac{\card{B}}{\card{A}\sqz{+}\card{B}}$} \var{\vectbf{x}|B} \nonumber \\
& \hspace{16mm} \sqz{+} \scalebox{1}{$\frac{\card{A}\card{B}}{\prl{\card{A}\sqz{+}\card{B}}^2}$} \!\norm{ \big. \ctrd{\vectbf{x}|A} \sqz{-} \ctrd{\vectbf{x}|B}}_2^2, 
\end{align}
\end{subequations}
}
where for any $I \subseteq \indexset\;$ $\ctrd{\vectbf{x}|I}$\refeqn{eq.centroid} and $\var{\vectbf{x}|I}$\refeqn{eq.variance} denote the centroid and variance of $\vectbf{x}|I$, respectively.  
Note that any singleton cluster $i \in \cluster{\treetop}$ has $\card{i} = 1$, $\ctrd{\vectbf{x}|i} = \vect{x}_i$ and $\var{\vectbf{x}|i} = 0$.
Also, note that after an NNI restructuring of a clustering tree, the data structure keeping cluster sizes, means and variances can be updated in constant time using \refeqn{eq.SufficientStatisticRecursion}. 
Therefore, given the sufficient statistics, local homogeneity of a clustering hierarchy at any cluster with respect to Ward's linkage can be determined in constant time.

To demonstrate another computationally efficient setting of anytime hierarchical clustering, consider the squared Euclidean distance as a dissimilarity measure, i.e. for any $\vect{x}, \vect{y} \in \R^m$
\begin{align} \label{eq.SquaredEuclidean}
\dist\prl{\vect{x}, \vect{y}} = \norm{\vect{x} - \vect{y}}_2^2.
\end{align}  
As shown in \refapp{app.SpecialAverageLinkage}, for any $\vectbf{x} \in \prl{\R^m}^{\indexset}$ and disjoint subsets $A,B \subseteq \indexset$, the average linkage $\linkage_{A}$ \refeqn{eq.AverageLinkage} between partial patterns $\vectbf{x}|A$ and $\vectbf{x}|B$, based on the squared Euclidean distance, can be rewritten in terms of sufficient statistics of $\vectbf{x}|A$ and $\vectbf{x}|B$ as \footnote{This is generally known as the ``bias-variance" decomposition of squared Euclidean distance \cite{hastie_tibshirani_friedman_2009}. }
\begin{align} 
\linkage_A\!\prl{\vectbf{x};\! A,\! B} \sqz{=} \var{\vectbf{x}|A} + \var{\vectbf{x}|B} + \norm{\ctrd{\vectbf{x}|A} - \ctrd{\vectbf{x}|B}}_2^2.  \label{eq.AverageLinkSimple}
\end{align}
Therefore, as in the case of Ward's linkage, given the sufficient statistics of a clustering hierarchy its local homogeneity at any cluster with respect to  average linkage with the squared Euclidean distance \refeqn{eq.SquaredEuclidean}   can be determined in constant time.

\begin{table}[tb]
\caption{Determining Local Homogeneity of a Clustering Hierarchy at any Cluster}
\centering
{\small
\begin{tabular}{|@{\hspace{0.25mm}}c@{\hspace{0.25mm}}|@{\hspace{0.25mm}}c@{\hspace{0.25mm}}|@{\hspace{0.25mm}}c@{\hspace{0.25mm}}|@{\hspace{0.25mm}}c@{\hspace{0.25mm}}|@{\hspace{0.25mm}}c@{\hspace{0.25mm}}|@{\hspace{0.25mm}}c@{\hspace{0.25mm}}|@{\hspace{0.25mm}}c@{\hspace{0.25mm}}|} 
\hline
\multirow{3}{*}{Linkage}& \multirow{3}{*}{Single} & \multirow{3}{*}{Complete} & \multirow{3}{*}{Average} & \multirow{3}{*}{Ward\footnotemark}  & \multirow{3}{*}{Minimax} & Average\addtocounter{footnote}{-1}\footnotemark \\
&&&&&& with\\
&&&&&& \refeqn{eq.SquaredEuclidean} \hspace{-2mm} or \hspace{-2mm}  \refeqn{eq.Cosine}\\
\hline
Complexity$\Big.$ & $\bigO{n^2}$ & $\bigO{n^2}$ & $\bigO{n^2}$  & $\bigO{1}$ & $\bigO{n^2}$ & $\bigO{1}$  \\
\hline 
\end{tabular}
\null \hfill {$n$: the number of data points \hspace{-4mm}}
}
\label{tab.Complexity}
\end{table}
\footnotetext{Assuming the availability of sufficient statistics.}

A similar computational improvement for average linkage is also possible with the cosine dissimilarity --- another commonly used dissimilarity, in information retrieval and text mining \cite{salton_buckley_IPM1988}: for any $\vect{x}, \vect{y} \in \R^m$,
\begin{align}\label{eq.Cosine}
\dist\prl{\vect{x}, \vect{y}} = 1 - \frac{\vect{x} \cdot \vect{y}}{\norm{\vect{x}}_2 \norm{\vect{y}}_2},
\end{align}    
where $\cdot$ denote the Euclidean dot product.
For any dataset of unit length vectors $\vectbf{x} \in \prl{\Sp^{m-1}}^{\indexset}$ and disjoint subsets $A,B \subseteq \indexset$, the average linkage $\linkage_A$ \refeqn{eq.AverageLinkage} between $\vectbf{x}|A$ and $\vectbf{x}|B$, based on the cosine dissimilarity, is given by
\begin{align}
\linkage_A\prl{\vectbf{x};A,B} = 1 - \ctrd{\vectbf{x}|A} \cdot \ctrd{\vectbf{x}|B}.
\end{align}   
\reftab{tab.Complexity} briefly summaries the discussion on computational complexity of the determination of local homogeneity of a clustering hierarchy at any cluster.

%%%%%%%%%%%%%%%%%%%%%%%%%%%%%%%%%%%%%%%%%%%
%%%%%%%%%%%%%%%%%%%%%%%%%%%%%%%%%%%%%%%%%%%
\subsection{Application: Incremental Clustering}
\label{sec.IncrementalClustering}
%\subsection{Handling Data Insertion \& Deletion}
%%%%%%%%%%%%%%%%%%%%%%%%%%%%%%%%%%%%%%%%%%%
%%%%%%%%%%%%%%%%%%%%%%%%%%%%%%%%%%%%%%%%%%%

As an application of  anytime  clustering, given a choice of linkage $\linkage$, we now propose an incremental hierarchical clustering method consisting of the following steps:  (i) insert a new data point to existing clustering hierarchy based on a specific tree traversal and local homogeneity criterion as described in \reftab{tab.Insert}, and then (ii) apply anytime clustering of \reftab{tab.OHC} to obtain a  homogeneous clustering hierarchy of the updated data set with respect to $\linkage$. 

\begin{table}[h]
\caption{Incremental Hierarchical $\linkage$-Clustering: Data Insertion Using Local Homogeneity}
\label{tab.Insert}
\begin{tabular}{|p{0.45\textwidth}|}
\hline
\\[-3mm]

Let  $\treetop \in \bintreetopspace_{\indexset}$ be a clustering hierarchy associated with a set of data points $\vectbf{x} \in \prl{\R^m}^{\indexset}$ and $\linkage$ be a linkage function.
Let $i \not \in \indexset$ denote the label of a new data point $\vect{x}_i \in \R^m$ to be inserted, and $\hat{\indexset} = \indexset \cup \crl{i}$ and $\hat{\vectbf{x}} = \prl{\vect{x}_j}_{j \in \hat{\indexset}}$ be the updated index and data sets after data insertion, respectively.

To insert $\vect{x}_i$ into the existing clustering hierarchy $\treetop$ associated with $\vectbf{x}$:
\begin{enumerate}
\item Start with $K = \indexset$.
\item For $\crl{K_L, K_R} = \childCL{K,\treetop}$,
\begin{enumerate}
\item If {\small $\linkage\prl{\hat{\vectbf{x}};\!K_L,\! K_R} \sqz{\leq} \min \prl{\big.\linkage\prl{\hat{\vectbf{x}}; \!K_L,\! \crl{i}}\!,\linkage\prl{\hat{\vectbf{x}}; \!K_R,\! \crl{i}}} $}, then\addtocounter{footnote}{-5}\footnotemark\addtocounter{footnote}{5}  \; attach leaf $i$ as the sibling of $K$ in the new clustering tree $\hat{\treetop} \in \bintreetopspace_{\hat{\indexset}}$.
\item Otherwise, set {\small$K \sqz{=} \argmin_{D \in \childCL{K,\treetop}} \linkage \prl{\hat{\vectbf{x}};\! D\!, \crl{i}\!}$}, and go to step 2.  
\end{enumerate} 
\vspace{-4mm}
\end{enumerate}
\\
\hline
\end{tabular}
\end{table}

%Suppose  $\treetop \in \bintreetopspace_{\indexset}$ be a clustering hierarchy associated with a set of data points $\vectbf{x} \in \prl{\R^m}^{\card{\indexset}}$ and a linkage function $\linkage$.
%Let $i \not \in \indexset$ denote the label of a new data point $\vect{x}_i \in \R^m$ to be inserted, and $\hat{\indexset} = \indexset \cup \crl{i}$ and $\hat{\vectbf{x}} = \prl{\vect{x}_j}_{j \in \hat{\indexset}}$ be the updated index and data sets after data insertion, respectively.
%Based on structural homogeneity of clustering hierarchies, we propose the following procedure for data insertion into the existing clustering hierarchy $\treetop \in \bintreetopspace_{\indexset}$ associated with $\vectbf{x}$ and $\linkage$: 
%%
%\begin{enumerate}
%\item Start with $K = \indexset$.
%\item For $\crl{K_L, K_R} = \childCL{K,\treetop}$,
%\begin{enumerate}
%\item If $\linkage\prl{\hat{\vectbf{x}};K_L, K_R} \leq \min \prl{\linkage\prl{\hat{\vectbf{x}}; K_L, \crl{i}}, \linkage\prl{\hat{\vectbf{x}}; K_L, \crl{i}}} $, then\addtocounter{footnote}{-6}\footnotemark\addtocounter{footnote}{6}  \; attach leaf $i$ as the sibling of $K$ in the new clustering tree $\hat{\treetop} \in \bintreetopspace_{\hat{\indexset}}$.
%\item Otherwise, set $K = \argmin_{D \in \childCL{K,\treetop}} \linkage \prl{\hat{\vectbf{x}}; D, \crl{i}}$, and go to step 2.  
%\end{enumerate} 
%\end{enumerate}
%
Note that, given the linkage values of a clustering hierarchy, for any linkage function satisfying the recurrence formula \refeqn{eq.Recurrence} of Lance and Williams \cite{lance_williams_CJ1967} a data insertion, described in \reftab{tab.Insert}, can be performed in linear, $\bigO{n}$, time with the dataset size, $n$. 
This follows because the linkage distance between the new data point and clusters of an existing hierarchy can be efficiently computed in linear time using a post-order traversal of the clustering hierarchy and \refeqn{eq.Recurrence}.

%%%%%%%%%%%%%%%%%%%%%%%%%%%%%%%%%%%%%%%%%%%
%%%%%%%%%%%%%%%%%%%%%%%%%%%%%%%%%%%%%%%%%%%
\section{Experimental Evaluation}
\label{sec.ExperimentalEvaluation}
%%%%%%%%%%%%%%%%%%%%%%%%%%%%%%%%%%%%%%%%%%%
%%%%%%%%%%%%%%%%%%%%%%%%%%%%%%%%%%%%%%%%%%%

This section presents a preliminary comparative numerical study of three different  hierarchical clustering methods using both simulated and real datasets. 
We compare: (a)  the standard agglomerative batch method ($\HAC_{\linkage}$, \reftab{tab.HAC}); with (b) the new anytime method ($\AHC_{\linkage}$, \reftab{tab.OHC}) and (c) its specialization to the incremental ``data insertion" problem setting ($\IHC_{\linkage}$,\reftab{tab.Insert}). 

%This section presents an empirical comparison of agglomerative, incremental and online  hierarchical clustering methods using both simulated and real datasets. 

%%%%%%%%%%%%%%%%%%%%%%%%%%%%%%%%%%%%%%%%%%%%
%%%%%%%%%%%%%%%%%%%%%%%%%%%%%%%%%%%%%%%%%%%%
\subsection{Datasets}
\label{sec.Dataset}
%%%%%%%%%%%%%%%%%%%%%%%%%%%%%%%%%%%%%%%%%%%%
%%%%%%%%%%%%%%%%%%%%%%%%%%%%%%%%%%%%%%%%%%%%

Very high dimensional and sparse data sets generally have simple structure and, specifically, are known to tend toward ultrametricity\footnote{A metric $\dist:X \times X \rightarrow \R_{\geq 0}$ is said to be a \emph{ultrametric} if it satisfies the strong triangle inequality, i.e. for any $\vect{x}, \vect{y}, \vect{z} \in X$, $\dist\prl{\vect{x}, \vect{y}} \leq \max\prl{\dist\prl{\vect{x}, \vect{z}}, \dist\prl{\vect{z}, \vect{y}}}$.} with the increasing dimensionality and/or sparsity\cite{murtagh_JC2009}.
In this context, the fact that a monotone clustering hierarchy associated with a dataset defines an ultrametric between data points (as we will briefly review in the next section) \cite{carlsson_memoli_jmlr2010}, motivates the intuition that hierarchical methods may enjoy particular efficacy in  clustering problems involving high dimensional and sparse data.
In the following preliminary study we will compare the results of hierarchical clustering on a low dimensional synthetic dataset and a higher dimensional dataset of physical origin. 
In both cases we will use a validation measure (introduced below) that quantifies the loss of information incurred by  approximating the underlying pairwise dissimilarities between points with the coarsened measure arising from the ultrametric induced by the resulting cluster hierarchy.

A challenging dataset for any hierarchical clustering method consists of uniformly distributed low dimensional data points. 
We generate our synthetic data by uniformly sampling the planar unit cell,
$\brl{0,1} \times \brl{0,1}$, thereby generating similar populations of varied cardinality. 
In contrast, for real data points, we use the MNIST collection of handwritten digits, where each data sample is a black and white
$28\times28$ image of a human produced numeral \cite{lecunEtAl1998}.   
We generate test datasets of varied size by randomly sampling an equal number of images for each digit in  the MNIST dataset.

%%%%%%%%%%%%%%%%%%%%%%%%%%%%%%%%%%%%%%%%%%
%%%%%%%%%%%%%%%%%%%%%%%%%%%%%%%%%%%%%%%%%%
\subsection{Validation Measure}
\label{sec.Validation}
%%%%%%%%%%%%%%%%%%%%%%%%%%%%%%%%%%%%%%%%%%
%%%%%%%%%%%%%%%%%%%%%%%%%%%%%%%%%%%%%%%%%%

\begin{figure}[tb] 
\centering
\hspace{-0.5mm}\includegraphics[width=0.24\textwidth]{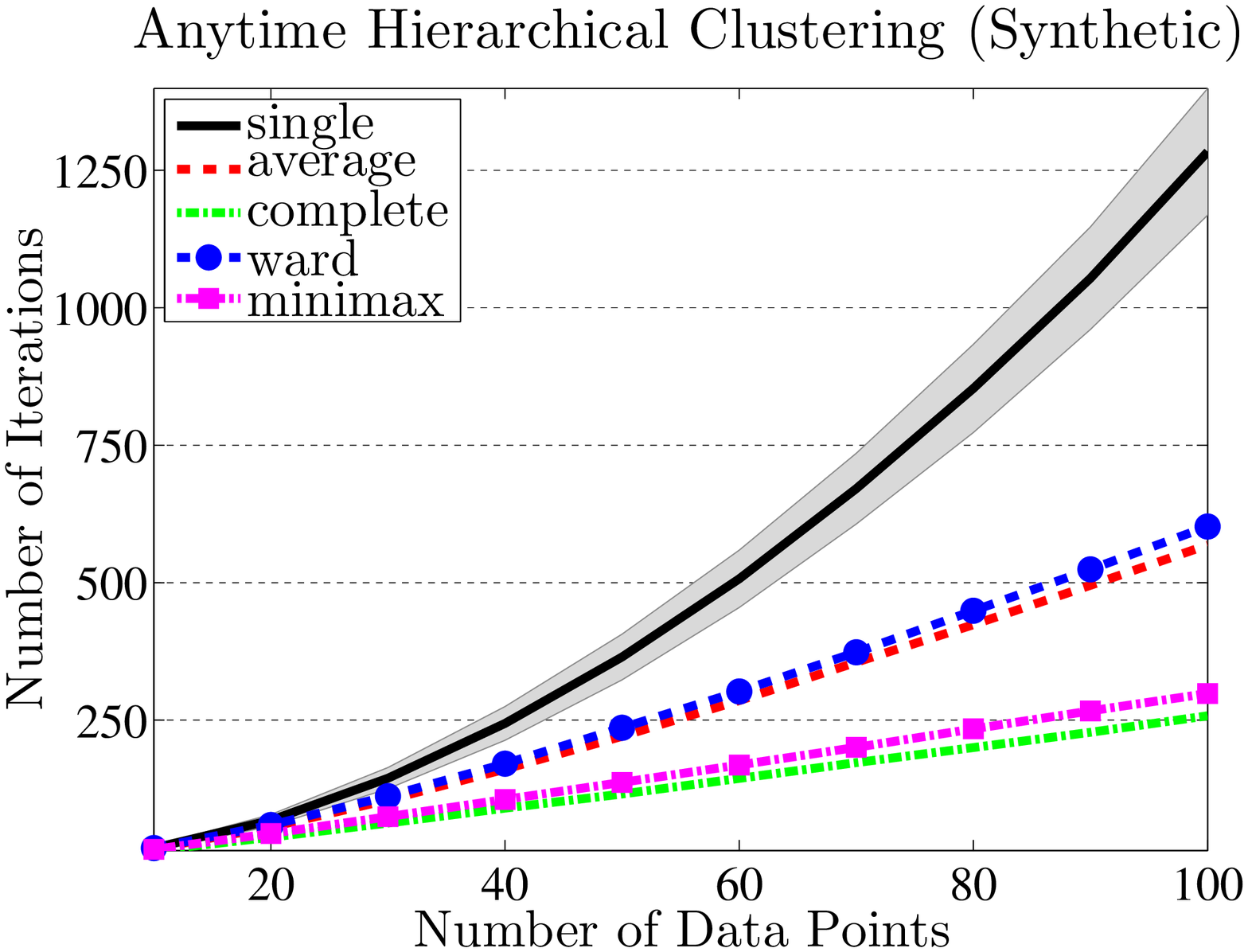}  \hspace{-2.5mm}
\includegraphics[width=0.24\textwidth]{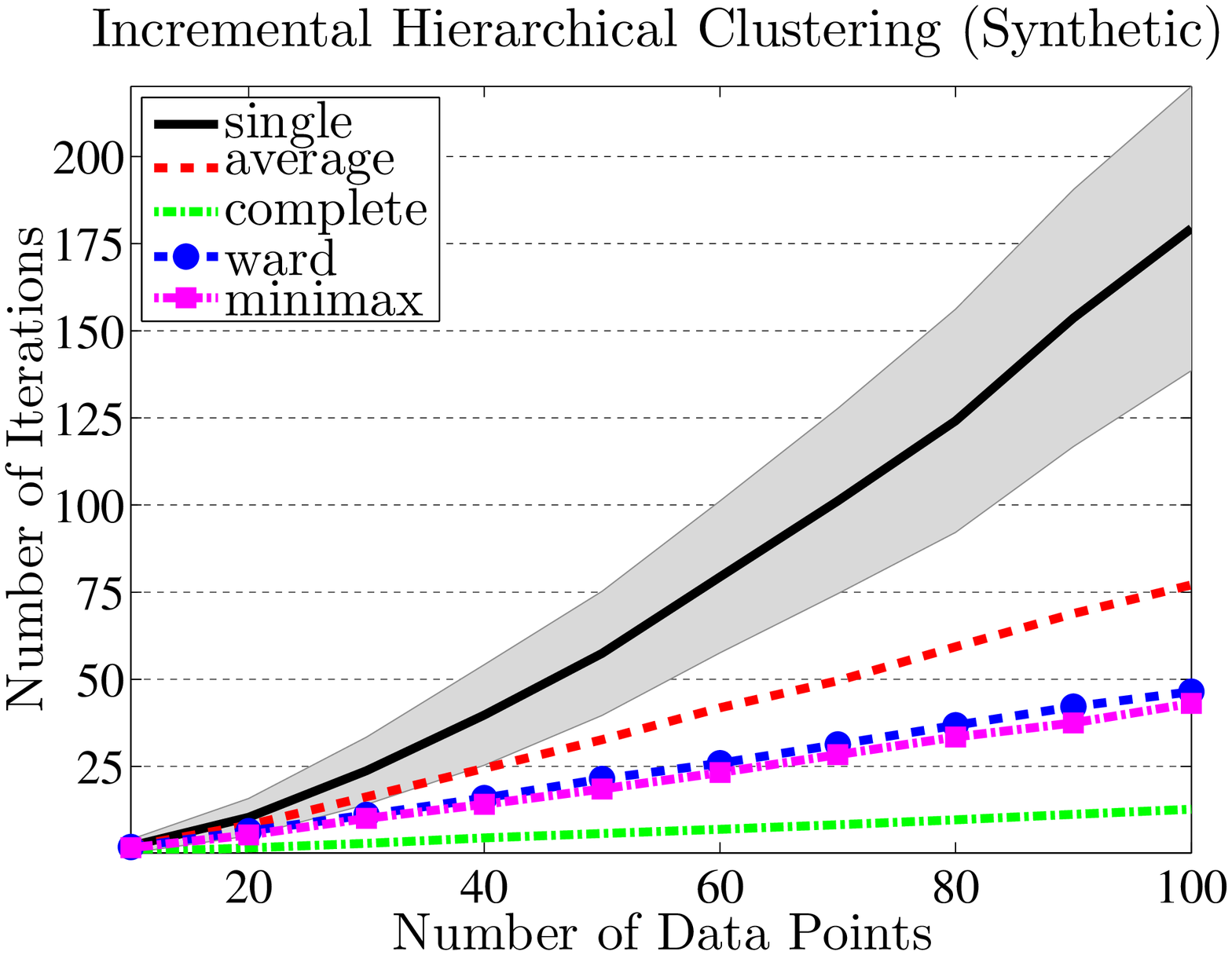} 
\caption[The LOF caption]{Average number of iterations of anytime (left) and incremental (right) hierarchical clusterings of simulated data points.
Shaded regions illustrate the sample variance for the single linkage.\footnotemark} 
\label{fig.NumIter_Synthetic}
\end{figure}
\footnotetext{To prevent cluttered figures and give some idea of how the sample variance changes with the cardinality of dataset, we only include the results for single linkage .}

To evaluate the accuracy and effectiveness of different hierarchical clustering methods, we use the cophenetic correlation coefficient --- a widely accepted validation criterion  that measures how well a clustering hierarchy preserves the underlying pairwise dissimilarities between points in a dataset\cite{sokal_rohlf_T1962}. 
In order to  interpret this criterion we find it helpful to briefly review the manner in which a monotone hierarchy induces an ultrametric between points in a dataset \cite{carlsson_memoli_jmlr2010}.

For any set of data points $\vectbf{x} \in \prl{\R^m}^{\indexset}$ and a clustering hierarchy $\treetop \in \bintreetopspace_{\indexset}$ associated with $\vectbf{x}$ and a linkage $\linkage$, let $\mat{D}\prl{\vectbf{x}} \in \R^{\indexset} \times \R^{\indexset}$ and $\mat{U}_{\treetop}\prl{\vectbf{x}}\in \R^{\indexset} \times \R^{\indexset}$ denote the original distance matrix of $\vectbf{x}$  and induced ultrametric of $\treetop$, respectively.
Namely, for any $I \in \cluster{\treetop}$, $i \in I$ and $j \in \compLCL{I}{\treetop}$ 
\begin{align}
\mat{D}\prl{\vectbf{x}}_{ij} &= \norm{\vect{x}_i - \vect{x}_j}_2, \\
\mat{U}_{\treetop}\prl{\vectbf{x}}_{ij} &=  \linkage\prl{\vectbf{x}; I, \compLCL{I}{\treetop}},
\end{align} 
and for any $i \in \indexset$ set $\mat{D}\prl{\vectbf{x}}_{ii} = \mat{U}\prl{\treetop}_{ii} = 0$.  
The cophenetic correlation coefficient between $\mat{D}\prl{\vectbf{x}}$ and $\mat{U}_{\treetop}\prl{\vectbf{x}}$ is defined as
\begin{align}\label{eq.Cophenetic}
\rho_c = \frac{\sum_{i,j \in \indexset}\prl{\mat{D}_{ij} -\overline{\mat{D}}}\prl{\mat{U}_{ij} - \overline{\mat{U}}} }{\sqrt{\sum_{i,j \in \indexset} \prl{\mat{D}_{ij} -\overline{\mat{D}}}^2}\sqrt{\sum_{i,j \in \indexset} \prl{\mat{U}_{ij} -\overline{\mat{U}}}^2}},
\end{align}
where $\overline{\mat{D}}$ and $\overline{\mat{U}}$ denote the average of the elements of $\mat{D}$ and $\mat{U}$, respectively, i.e. $\overline{\mat{D}} = \frac{1}{\card{\indexset}^2} \sum_{i,j \in \indexset} \mat{D}_{ij}$.

Finally, it is useful to note that  Ward's linkage $\linkage_W$ \refeqn{eq.WardLinkage} quantifies the change in the sum of squared error after merging clusters \cite{ward_JASS1963}, and so $\linkage_W$  and the standard Euclidean norm do not have the same units. 
To resolve this unit mismatch, for any clustering hierarchy $\treetop$ resulting from hierarchical clustering of $\vectbf{x}$ based on Ward's linkage $\linkage_{W}$ we find it convenient to use average linkage $\linkage_{A}$ \refeqn{eq.AverageLinkage} to define the induced dissimilarity $\mat{U}\prl{\treetop}$ of $\treetop$.

%%%%%%%%%%%%%%%%%%%%%%%%%%%%%%%%%%%%%%%%%%
%%%%%%%%%%%%%%%%%%%%%%%%%%%%%%%%%%%%%%%%%%
\subsection{Preliminary Numerical Results}
\label{sec.Result}
%%%%%%%%%%%%%%%%%%%%%%%%%%%%%%%%%%%%%%%%%%
%%%%%%%%%%%%%%%%%%%%%%%%%%%%%%%%%%%%%%%%%%

\begin{figure}[tb] 
\centering
\hspace{-0.5mm}\includegraphics[width=0.24\textwidth]{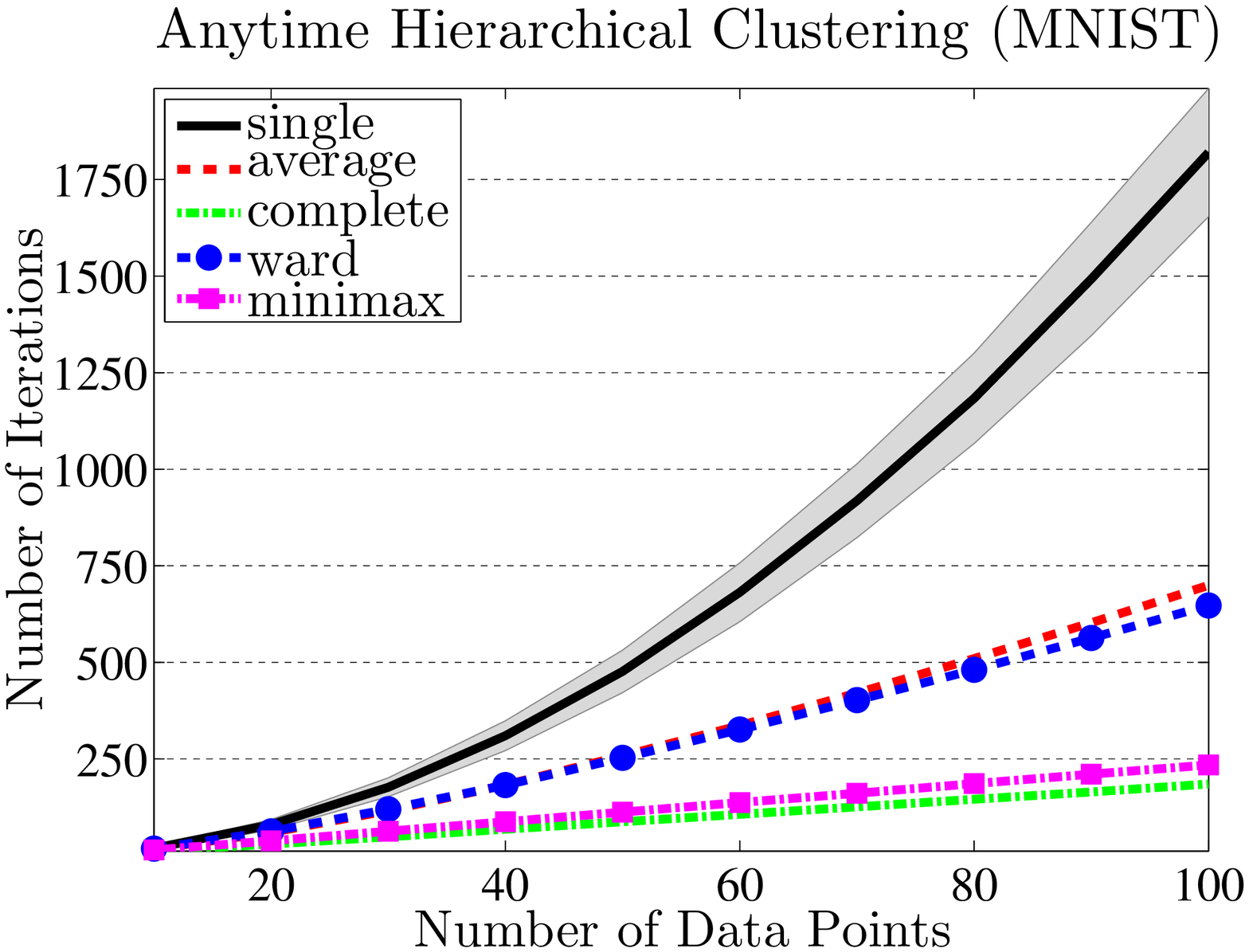}  \hspace{-2.5mm}
\includegraphics[width=0.24\textwidth]{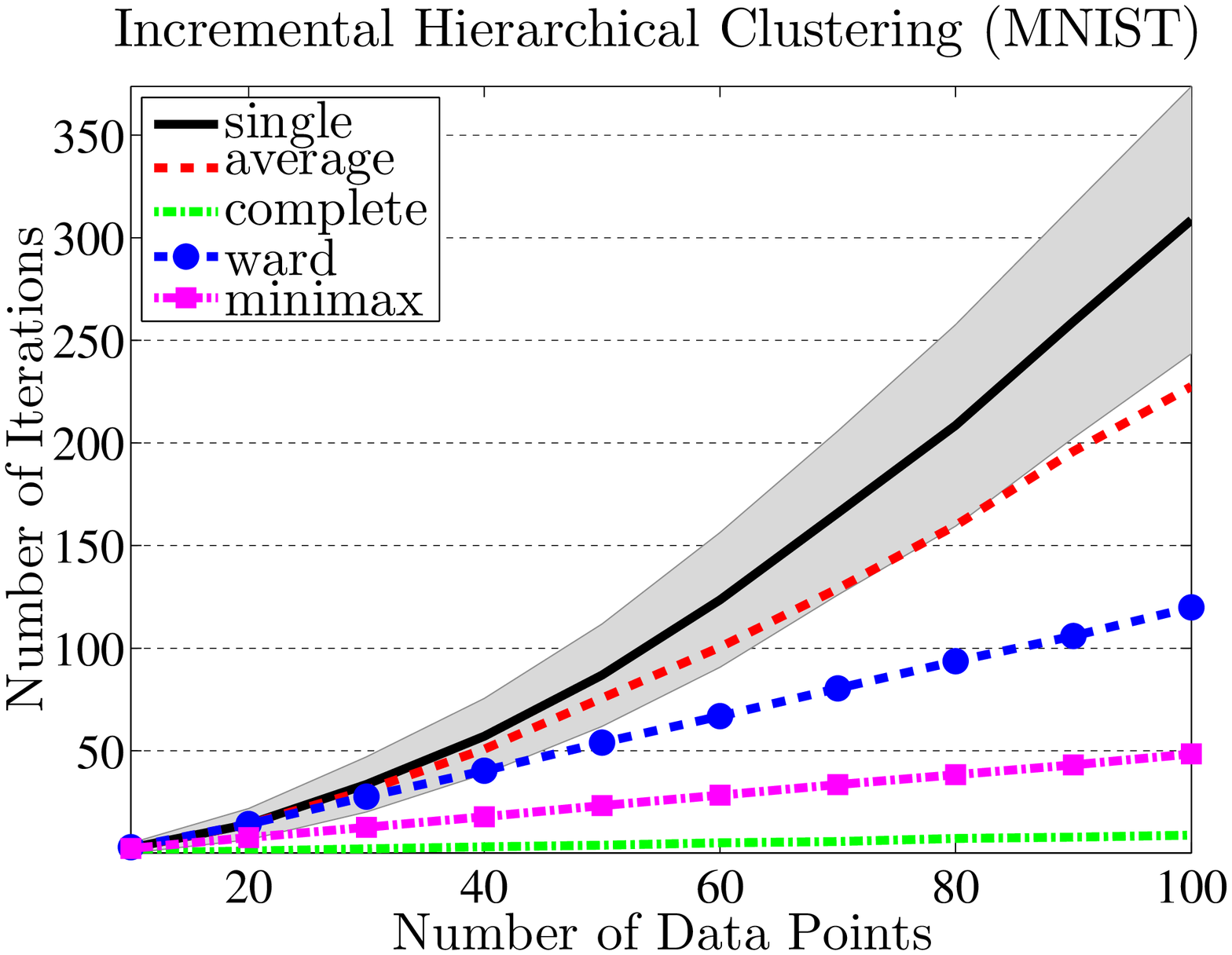} 
\caption[The LOF caption]{Average number of iterations of anytime (left) and incremental (right) hierarchical clusterings for the MNIST dataset. Shaded regions illustrate the sample variance for the single linkage.\addtocounter{footnote}{-1}\footnotemark} 
\label{fig.NumIter_MNIST}
\end{figure}

Using an empirical evaluation of anytime and incremental hierarchical clustering  methods we aim to statistically explore: (1) the number of iterations of anytime and incremental clusterings, and (2) their effectiveness compared to the traditional agglomerative clustering methods.

As expected, the number of iterations to homogeneous termination of any anytime clustering depends strongly on the initial conditions (i.e, the initial pair of dataset and tree).
To challenge the proposed clustering method,  we always start anytime clustering of a dataset at  a random initial clustering hierarchy  uniformly sampled from the space of non-degenerate hierarchies  \cite{semple_steel_2003}.
To give some preliminary idea of performance as a function of data size we run the various methods on datasets of cardinality $10, 20, \ldots, 100,$ and report statistics from the results of $1000$ different randomly selected pairings of initial data set and tree using both synthetic and real data collections generated as described in \refsec{sec.Dataset}.

\begin{figure*}[thb]
\centering
\includegraphics[width=0.3\textwidth]{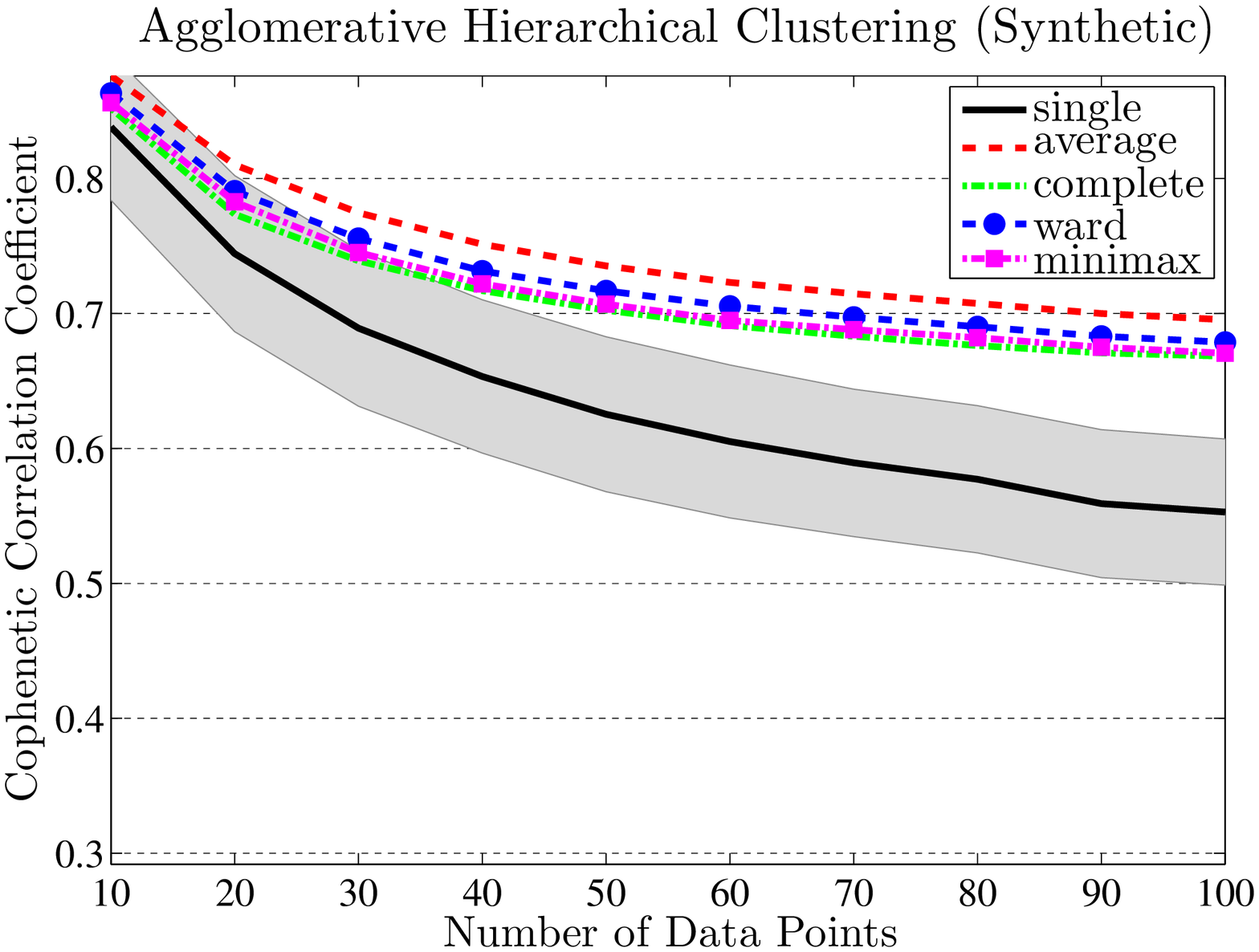} \hspace{2mm}
\includegraphics[width=0.3\textwidth]{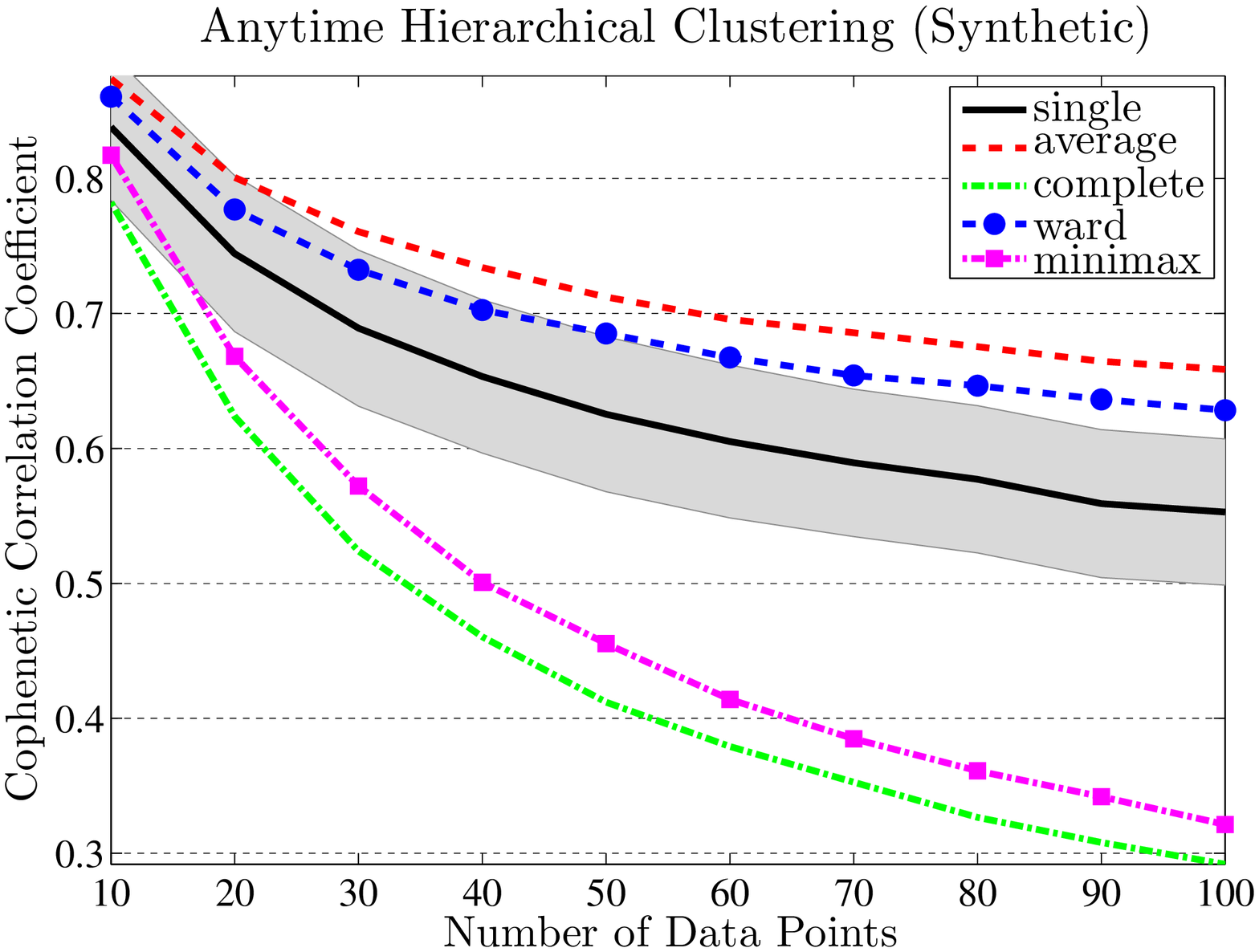} \hspace{2mm}
\includegraphics[width=0.3\textwidth]{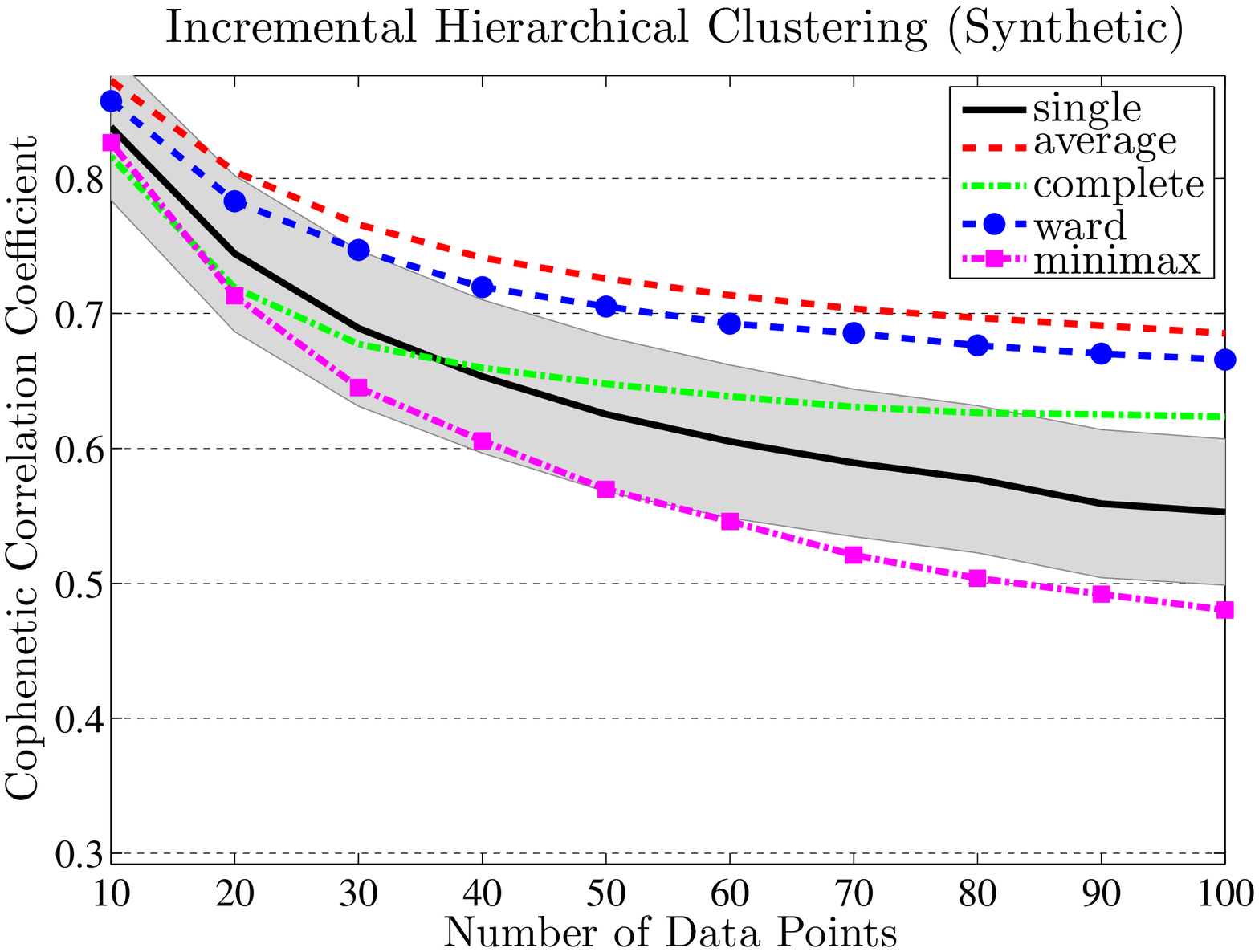} 
\caption[The LOF caption]{Average cophenetic correlation coefficient for agglomerative (left), anytime (middle) and incremental (right) hierarchical clusterings of simulated data points. Shaded regions illustrate the sample variance for the single linkage.\addtocounter{footnote}{-1}\footnotemark}
\label{fig.CCC_Synthetic}
\end{figure*}

\begin{figure*}[tbh]
\centering
\includegraphics[width=0.3\textwidth]{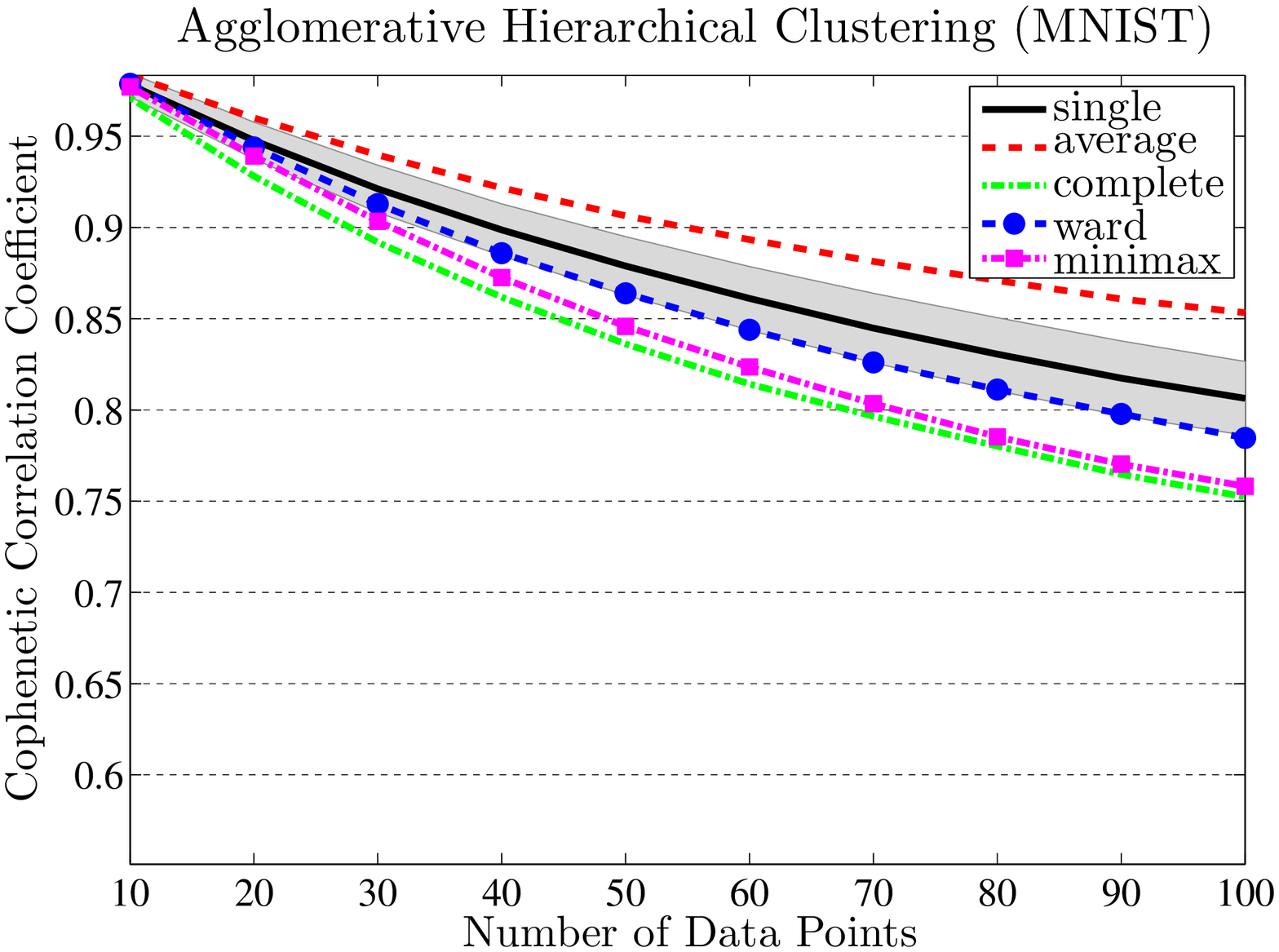} \hspace{2mm}
\includegraphics[width=0.3\textwidth]{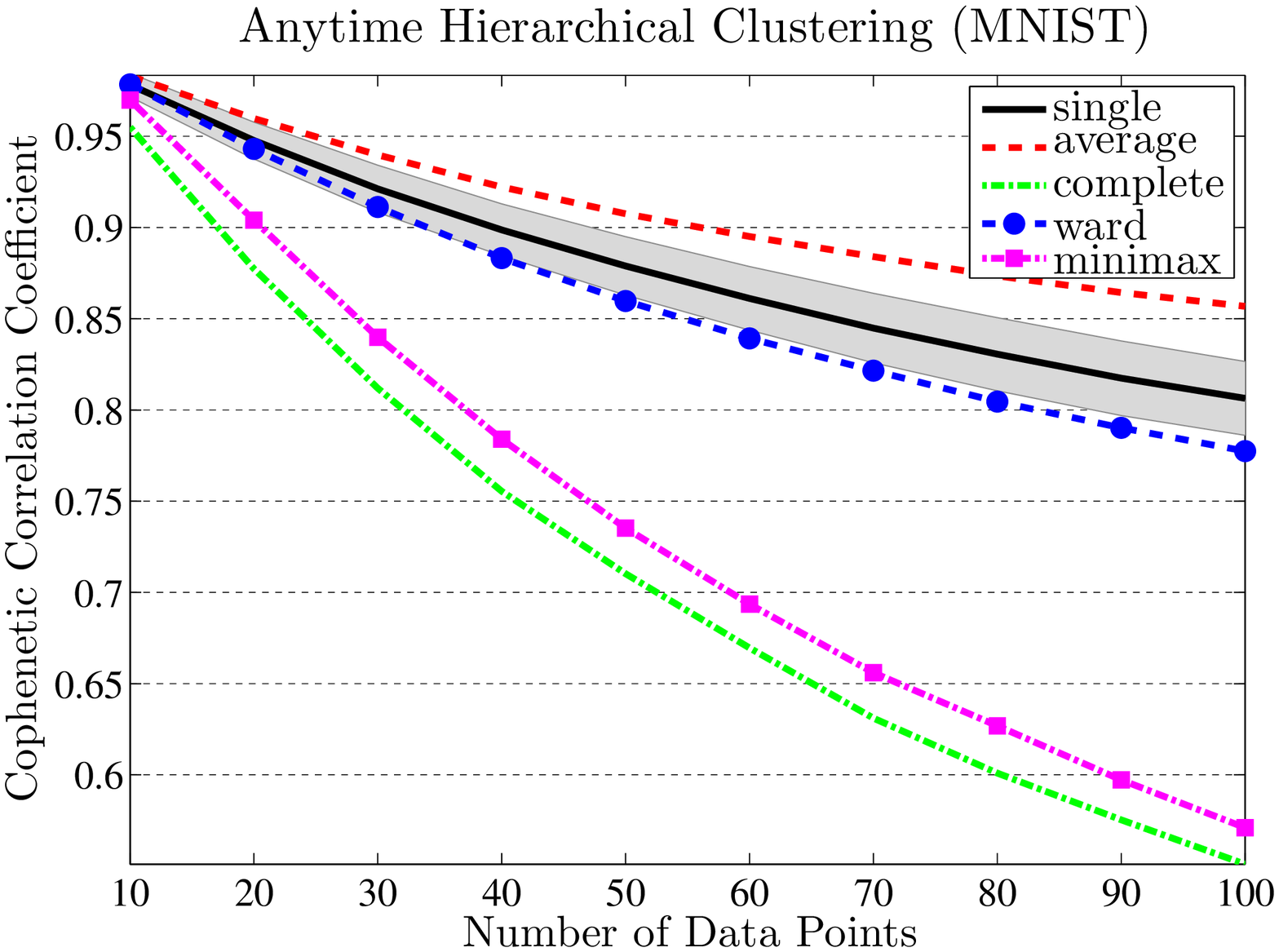} \hspace{2mm}
\includegraphics[width=0.3\textwidth]{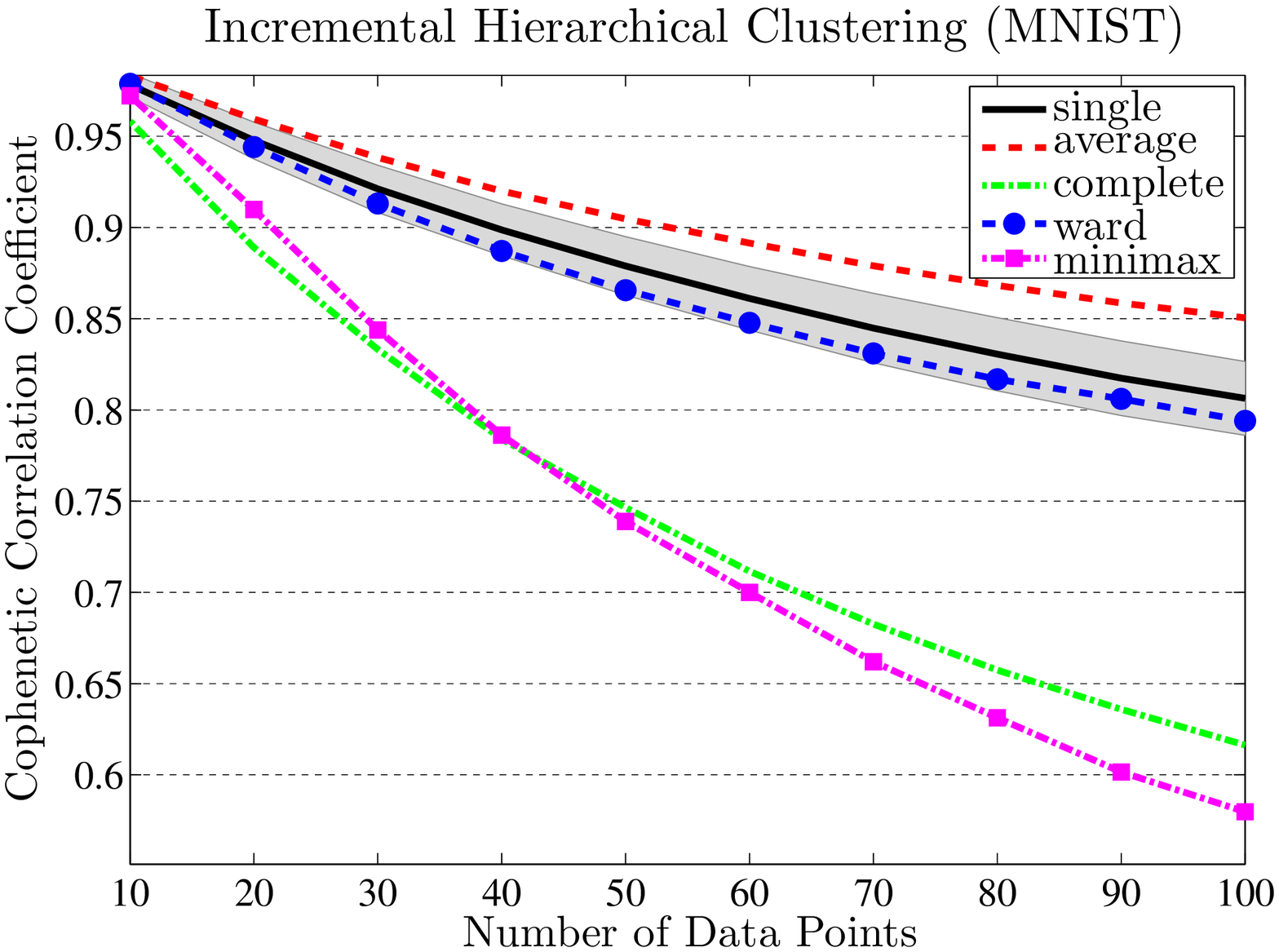} 
\caption[The LOF caption]{Average cophenetic correlation coefficient for agglomerative (left), anytime (middle) and incremental (right) hierarchical clusterings for the MNIST dataset. Shaded regions illustrate the sample variance for the single linkage.\addtocounter{footnote}{-1}\footnotemark}
\label{fig.CCC_MNIST}
\end{figure*}

For each linkage function discussed in \refsec{sec.linkage}, \reffig{fig.NumIter_Synthetic} and \reffig{fig.NumIter_MNIST} present the average number of iterations of anytime and incremental clusterings versus the dataset size. 
Regardless of  linkage and size, incremental clustering generally requires an order of magnitude fewer iterations than does anytime clustering. 
This is a consequence of our experimental design whereby anytime clustering is always initialized from a random clustering tree
while incremental clustering takes the advantage of local homogeneity
for effective insertion of a new datum into an existing clustering
tree. 
The next clearest pattern that emerges from these figures is that the number of iterations for both anytime and incremental clustering methods seems to grow quite differently for single linkage (where it appears quadratic in the size of the dataset) than for other familiar linkages (relative to which  these preliminary statistics are not inconsistent with linear growth) - but more exploration with larger cardinality datasets will be required before more specific conjectures are possible.

\reffig{fig.CCC_Synthetic} and \reffig{fig.CCC_MNIST} illustrate how cophenetic correlation coefficient \refeqn{eq.Cophenetic} changes with  the dataset size for agglomerative, anytime and incremental hierarchical clustering methods.
Recall from \refthm{thm.SingleLinkage} that a clustering hierarchy resulting from agglomerative single linkage clustering  of a dataset is uniquely characterised by its homogeneity relative to single linkage.
Hence, it must be the case (as is indeed reflected in these figures) that the clustering performance of single linkage clustering is the same for all agglomerative, anytime and incremental methods.
Here, assuming the results of agglomerative clustering as the ground truth,  anytime  and incremental clusterings using complete and minimax linkages is observed to perform relatively poorly, which is probably due to their overestimation of cluster dissimilarities. 
On the other hand, for average and Ward's linkages anytime and incremental hierarchical clusterings   perform as well as  agglomerative clustering. 
Further, as expected,  incremental clustering generally performs better than anytime clustering since  it  uses local  homogeneity to properly insert each new data point and calls anytime clustering  starting at a clustering tree which is far better than a random hierarchy.
Finally, one can notice that the clustering performance of all hierarchical methods is better on real datasets than synthetic datasets, which is likely due to increased dimensionality and sparsity of data as discussed in \refsec{sec.Dataset}.

%%%%%%%%%%%%%%%%%%%%%%%%%%%%%%%%%%%%%%%%%%%
%%%%%%%%%%%%%%%%%%%%%%%%%%%%%%%%%%%%%%%%%%%
\section{Conclusions}
\label{sec.Conclusion}
%%%%%%%%%%%%%%%%%%%%%%%%%%%%%%%%%%%%%%%%%%%
%%%%%%%%%%%%%%%%%%%%%%%%%%%%%%%%%%%%%%%%%%%

In this paper, we introduce a new homogeneity criterion (\refdef{def.LocalStructuralOptimality}), for a clustering tree associated with a data set applicable to a reasonably broad subclass of the familiar  linkage functions.
We  show that homogeneity is a characteristic property of trees resulting from any such standard (linkage based) hierarchical clustering methods (\refprop{prop.HAC2LSO}). 
In particular, homogeneity uniquely characterizes the single linkage clustering tree of a data set (\refthm{thm.SingleLinkage}).

We propose an anytime hierarchical clustering method in \reftab{tab.OHC} that iteratively transforms any initial clustering hierarchy into a homogeneous clustering tree of a dataset relative to a user-specified linkage function. 
For the subclass of linkages (specified in \refdef{def.Contraction} ---  including single, complete, minimax and Ward's linkages) we demonstrate that this iterative clustering procedure must terminate in finite time (\refthm{thm.Convergence}). 
Finally , we discuss certain settings for computationally efficient anytime clustering and describe an incremental hierarchical clustering method based on local homogeneity of cluster trees and anytime clustering.

The ``anytime" nature of our  method enables
users to choose between accuracy and efficiency. 
In contrast to batch methods, any intermediate stage of anytime
hierarchical clustering returns a valid clustering tree and
incrementally improves the homogeneity at each iteration.
Thus, at any time a user can stop clustering and continue
it later with or without updating the data set. 
Further, since our method is based on local tree  restructuring (the familiar NNI-walk, \refdef{def:NNIMove} \cite{robinson_jct1971,moore_goodman_barnabas_jtb1973}), it provides an opportunity for distributed/parallel implementation and reactive tracking.  
The experimental evaluation of sample anytime and incremental hierarchical clustering approaches suggested by these new ideas suggests their value relative to the standard ``batch"  methods.  

Work is presently in progress to  establish bounds on the number of iterations of anytime and incremental hierarchical clustering methods, and to develop specific implementations for efficient computation of anytime clustering. 
In the longer term, we believe these ideas will extend to a randomized algorithm for anytime single linkage clustering as well as to settings where simultaneous distance metric learning must take place in parallel with the hierarchical clustering process.

%%%%%%%%%%%%%%%%%%%%%%%%%%%%%%%%%%%%%%%%%%%%%%%%%%%%%%%%%%
%%%%%%%%%%%%%%%%%%%%%%%%%%%%%%%%%%%%%%%%%%%%%%%%%%%%%%%%%%
\section{Acknowledgement}
\label{sec.Acknowledgement}
%%%%%%%%%%%%%%%%%%%%%%%%%%%%%%%%%%%%%%%%%%%%%%%%%%%%%%%%%%
%%%%%%%%%%%%%%%%%%%%%%%%%%%%%%%%%%%%%%%%%%%%%%%%%%%%%%%%%%

This work was funded in part by the Air Force Office of Science Research under the MURI FA9550-10-1-0567.
\bibliographystyle{abbrv}
\bibliography{hierclustering}  % hierclustering.bib is the name of the Bibliography in this case
% ACM needs 'a single self-contained file'!
%\balancecolumns
\appendix
%Appendix A

%%%%%%%%%%%%%%%%%%%%%%%%%%%%%%%%%%%%%%%%%%%
%%%%%%%%%%%%%%%%%%%%%%%%%%%%%%%%%%%%%%%%%%%
%\section{Certain Classes of Linkages}
%\label{app.Linkage}
%%%%%%%%%%%%%%%%%%%%%%%%%%%%%%%%%%%%%%%%%%%
%%%%%%%%%%%%%%%%%%%%%%%%%%%%%%%%%%%%%%%%%%%
%
%\subsection{Examples of Reducible Linkages}
%\label{app.Reducible}
%%%%%%%%%%%%%%%%%%%%%%%%%%%%%%%%%%%%%%%%%%%
%%%%%%%%%%%%%%%%%%%%%%%%%%%%%%%%%%%%%%%%%%%
%
%%%%%%%%%%%%%%%%%%%%%%%%%%%%%%%%%%%%%%%%%%%
%%%%%%%%%%%%%%%%%%%%%%%%%%%%%%%%%%%%%%%%%%%
%\subsection{Examples of Strong Reducible Linkages}
%\label{app.StrongReducible}
%%%%%%%%%%%%%%%%%%%%%%%%%%%%%%%%%%%%%%%%%%%
%%%%%%%%%%%%%%%%%%%%%%%%%%%%%%%%%%%%%%%%%%%
%
%%%%%%%%%%%%%%%%%%%%%%%%%%%%%%%%%%%%%%%%%%%
%%%%%%%%%%%%%%%%%%%%%%%%%%%%%%%%%%%%%%%%%%%
%\subsection{Examples of NNI-Reducible Linkages}
%\label{app.NNIReducible}
%%%%%%%%%%%%%%%%%%%%%%%%%%%%%%%%%%%%%%%%%%%
%%%%%%%%%%%%%%%%%%%%%%%%%%%%%%%%%%%%%%%%%%%

%%%%%%%%%%%%%%%%%%%%%%%%%%%%%%%%%%%%%%%%%%%%%%%%%%%%%%
%%%%%%%%%%%%%%%%%%%%%%%%%%%%%%%%%%%%%%%%%%%%%%%%%%%%%%
\section{Sum of Ward's Linkages \& Sum of Squared Error}
\label{app.WardLinkage}
%%%%%%%%%%%%%%%%%%%%%%%%%%%%%%%%%%%%%%%%%%%%%%%%%%%%%%
%%%%%%%%%%%%%%%%%%%%%%%%%%%%%%%%%%%%%%%%%%%%%%%%%%%%%%

Although sums of linkage values $H_{\vectbf{x}, \linkage}$ \refeqn{eq.totallinkagecost} of distinct clustering hierarchies associated with a set of data points $\vectbf{x} \in \prl{\R^m}^{\indexset}$ and a linkage function $\linkage$  generally differ, $H_{\vectbf{x}, \linkage_{W}}$ is constant for Ward's linkage $\linkage_W$ \refeqn{eq.WardLinkage}:

\begin{lemma}
The sum of linkage values $H_{\vectbf{x}, \linkage_{W}}$\refeqn{eq.totallinkagecost} of any clustering hierarchy associated with a data set $\vectbf{x} \in \prl{\R^m}^{\indexset}$  and Ward's linkage $\linkage_{W}$ \refeqn{eq.WardLinkage} is constant and equal to the sum of squared errors of $\vectbf{x}$, i.e. for any $\treetop \in \bintreetopspace_{\indexset}$
\begin{align}
H_{\vectbf{x}, \linkage_{W}} \prl{\treetop} = \SSE{\vectbf{x}},
\end{align}
where
\begin{align}
\SSE{\vectbf{x}} \ldf \sum_{i \in \indexset} \norm{\vect{x}_i - \ctrd{\vectbf{x}|\indexset}}_2^2.
\end{align}
\end{lemma} 
\begin{proof}
Recall from \cite{ward_JASS1963,jain_dubes_1988} that Ward's linkage quantifies the change in the sum of squared errors after merging a group of data points, i.e. for any $\vectbf{x} \in \prl{\R^{m}}^{\indexset}$ and disjoint subsets $A,B \subseteq \indexset$
\begin{align} \label{eq.WardLinkageSSE}
\linkage_{W}\prl{\vectbf{x}; A,B} = \SSE{\vectbf{x}|A \sqz{\cup} B} \sqz{-} \SSE{\vectbf{x}|A} \sqz{-} \SSE{\vectbf{x}|B}\!.\!\!
\end{align}

Let $\treetop \in \bintreetopspace_{\indexset}$ be a binary hierarchy with the root split $\crl{\indexset_L, \indexset_R} = \childCL{\indexset,\treetop}$.
Proof by induction. 
\begin{itemize}
\item (Base Case) if $\card{\indexset} = 2$, then there is only one clustering hierarchy $\treetop \in \bintreetopspace_{\brl{2}}$, i.e. $\card{\bintreetopspace_{\brl{2}}} = 1$. 
Note that $\treetop$ only has one Ward's linkage joining two data points of $\vectbf{x} = \prl{\vect{x}_1, \vect{x}_2}$ whose value equals to the sum of squared errors of $\vectbf{x}$,
\begin{align}
H_{\vectbf{x}, \linkage_{W}}\prl{\treetop} 
&= \linkage_W\prl{\vectbf{x}; \crl{1}, \crl{2} } 
= \frac{1}{2} \norm{\vect{x}_1 - \vect{x}_2}_2^2, \\ 
&= \!\!\sum_{i \in \crl{1,2}} \norm{\vect{x}_i - \ctrd{\vectbf{x}|\crl{1,2}}}_2^2 = \SSE{\vectbf{x}}\!.\!\! 
\end{align}
\item (Induction) Let $\treetop_L$ and $\treetop_R$ denote the subtrees of $\treetop$ rooted at $\indexset_L$ and $\indexset_R$, respectively. Suppose that
\begin{align}
H_{\vectbf{x}|\indexset_L, \linkage_{W}}\prl{\treetop_L} = \SSE{\vectbf{x}|\indexset_L},\\
H_{\vectbf{x}|\indexset_R, \linkage_{W}}\prl{\treetop_R} = \SSE{\vectbf{x}|\indexset_R}.
\end{align}
Note that if any of subtrees only has one leaf, e.g. $\card{\indexset_L} = 1$ , then we set the associated sum of linkage values to zero, $H_{\vectbf{x}|\indexset_L, \linkage_{W}}\prl{\treetop_L} = 0$.
 
Hence, using \refeqn{eq.WardLinkageSSE}, one can obtain the result as follows:
\begin{align}
\hspace{-8.5mm}H_{\vectbf{x}, \linkage_{W}}\!\prl{\treetop} &\sqz{=} H_{\vectbf{x}|\indexset_L\!, \linkage_W}\!\prl{\treetop_L} \sqz{+} H_{\vectbf{x}|\indexset_R\!, \linkage_W}\!\prl{\treetop_R} \sqz{+} \linkage_{W}\!\prl{\vectbf{x};\! \indexset_L,\! \indexset_R}\!,\!\!\\
&\sqz{=} \SSE{\vectbf{x}|\indexset_L} \sqz{+} \SSE{\vectbf{x}|\indexset_R} +  \linkage_{W}\!\prl{\vectbf{x};\! \indexset_L,\! \indexset_R}\!,\!\!\\
&\sqz{=} \SSE{\vectbf{x} }.
\end{align}
\end{itemize}
\end{proof}

%%%%%%%%%%%%%%%%%%%%%%%%%%%%%%%%%%%%%%%%%%%%%%%%%%%
%%%%%%%%%%%%%%%%%%%%%%%%%%%%%%%%%%%%%%%%%%%%%%%%%%%
\section{Termination Analysis for \\ Average Linkage}
\label{app.AverageLinkageTermination}
%%%%%%%%%%%%%%%%%%%%%%%%%%%%%%%%%%%%%%%%%%%%%%%%%%%
%%%%%%%%%%%%%%%%%%%%%%%%%%%%%%%%%%%%%%%%%%%%%%%%%%%

\begin{lemma} \label{lem.AverageConvergence}
Iterated application of Anytime Hierarchical $\linkage_{A}$-Clustering procedure in \reftab{tab.OHC} initiated from any clustering tree in $\bintreetopspace_{\indexset}$ for a fixed set of data points $\vectbf{x} \in \prl{\R^m}^{\indexset}$  terminates in finite time at  a clustering hierarchy in $\bintreetopspace_{\indexset}$ that is homogeneous relative to average linkage $\linkage_A$ \refeqn{eq.AverageLinkage}.   
\end{lemma}
\begin{proof}
As in the proof of more general result in \refthm{thm.Convergence}, we shall show that the anytime clustering rule  does not cause any cycle in $\bintreetopspace_{\indexset}$ before terminating at a structurally homogeneous clustering hierarchy.
Consequently, the finite time termination of the anytime clustering method is simply due to finiteness of tree space $\bintreetopspace_{\indexset}$  \refeqn{eq.numBinTree}.

Let $S\prl{\treetop}$ denote  the ordered set of linkage values of a binary clustering hierarchy $\treetop \in \bintreetopspace_{\indexset}$  associated with $\vectbf{x}$ and $\linkage_{A}$ in ascending order, i.e.
\begin{align}
S\prl{\treetop} &\ldf \prl{\Big. \linkage_{A}\prl{\vectbf{x};  I,  \compLCL{I}{\treetop}}}_{   I  \in \cluster{\treetop}}, \\
S\prl{\treetop}_i &\leq  S\prl{\treetop}_j, \quad \forall \;1\leq i \leq j\leq 2\card{\indexset} -1, 
\end{align}
where $S\prl{\treetop} =\prl{\!S\prl{\treetop}_1\!, S\prl{\treetop}_2\!, \ldots, S\prl{\treetop}_{2\card{\indexset} -1}\! }$ and  note that the number of clusters of a binary tree over leaf set $\indexset$ is $2\card{\indexset} - 1$ \cite{schrijver2003}. 
Further, have the set of $\;\prl{2\card{\indexset} -1}$-tuple of real numbers ordered lexicographically according to the standard order of reals.

Let $\treetop^k \in \bintreetopspace_{\indexset}$ be a clustering hierarchy visited at $k$-th iteration of anytime hierarchical clustering of $\vectbf{x}$, where  $k \geq 0$.  
To prove the result, we shall show the following 
\begin{align}\label{eq.LinkageArrayMonotone}
S\prl{\treetop^k} > S\prl{\treetop^{k+1}}.
\end{align}

Since $\treetop^k$ and $ \treetop^{k+1}$ are NNI-adjacent, let $\prl{A^k,B^k,C^k}$ be the NNI-triplet (\reflem{lem.NNITriplet}) associated with the pair of $\prl{\treetop^k, \treetop^{k+1}}$.
Recall that $A^k \cup B^k \in \cluster{\treetop^k} \setminus \cluster{\treetop^{k+1}}$ and $B^k \cup C^k \in \cluster{\treetop^{k+1}} \setminus \cluster{\treetop^k}$.
 
Note that after the NNI transition from $\treetop^k$ to $\treetop^{k+1}$  two elements of $S\prl{\treetop^k}$, $\linkage_{A}\!\prl{\vectbf{x};\! A^k\!,\!B^k\!}$ and $\linkage_{A}\!\prl{\vectbf{x};\! A^k\sqz{\cup} B^k\!,\! C^k\!}$, are replaced by another two reals, $\linkage_{A}\!\prl{\vectbf{x};\! B^k\!,C^k\!}$ and $\linkage_{A}\!\prl{\vectbf{x};\! B^k\sqz{\cup} C^k \!, \! A^k\!}$, to yield $S\prl{\treetop^{k+1}}$.

By construction of the anytime clustering rule, we have 
\begin{align} \label{eq.NNIProp}
\linkage_{A}\prl{\vectbf{x}; A^k,B^k} &> \linkage_{A}\prl{\vectbf{x}; B^k,C^k}, \\
\linkage_{A}\prl{\vectbf{x}; A^k,C^k} &\geq \linkage_{A}\prl{\vectbf{x};B^k,C^k}, 
\end{align}
and, using the definition of average linkage \refeqn{eq.AverageLinkage}, one can verify that
\begin{align} \label{eq.ClusterProp}
\linkage_{A}\prl{\vectbf{x}; A^k \cup B^k, C^k} \geq \linkage \prl{\vectbf{x};B^k,C^k}.
\end{align}

If $\linkage_{A}\prl{\vectbf{x}; A^k \cup B^k, C^k} > \linkage_A \prl{\vectbf{x};B^k,C^k}
$, then it is clear  from \refeqn{eq.NNIProp} and \refeqn{eq.AverageLinkage} that
{\small 
\begin{align}
\min \prl{\big. \! \linkage_{A}\!\prl{\!\vectbf{x};\! A^k \!,\! B^k \!}\!, \linkage_{A}\!\prl{\!\vectbf{x};\! A^k  \sqz{\cup} B^k \!, \! C^k \!}\!\!} &\sqz{>}   \linkage_{A}\!\prl{\vectbf{x};\! B^k,\! C^k\!},\\
& \hspace{-21mm}\sqz{=} \min \prl{\big.\! \linkage_{A}\!\prl{\!\vectbf{x}; \!B^k \!, \! C^k \!}\!\!, \linkage_{A}\!\prl{\!\vectbf{x};\! B^k \sqz{\cup} C^k \!,\! A^k\!}\!\!}\!\!.\!\!
\end{align}
} 

Otherwise $\prl{\linkage_{A}\prl{\vectbf{x}; A^k \cup B^k, C^k} = \linkage_A \prl{\vectbf{x};B^k,C^k}}$, by definition \refeqn{eq.AverageLinkage}, we have $\linkage_A\prl{\vectbf{x}; B^k,C^k} = \linkage_{A}\prl{\vectbf{x}; A^k,C^k}$, and so $\linkage_{A} \prl{\vectbf{x}; A^k,B^k} > \linkage_{A}\prl{\vectbf{x}; B^k \cup C^k, A^k}$. 

In overall, the minimum of changed linkage values at each iteration of anytime clustering strictly decreases, which proves \refeqn{eq.LinkageArrayMonotone} and completes the proof.
%{\small
%\begin{align}
%&\min \prl{\Big. \!\! \crl{\big. \! \linkage_A\!\prl{\!\vectbf{x};\! A^k\!, \! B^k\!}\!\!, \linkage_A\!\prl{\!\vectbf{x}; \! A^k \sqz{\cup} B^k \!, \! C^k\!}\!\!} \sqz{\sqz{\setminus}} \crl{\!\!\big.\linkage_A\!\prl{\!\vectbf{x}; \!A^k \!, \!C^k\!}\!\!, \linkage_A\!\prl{\!\vectbf{x}; \!A^k \sqz{\cup} C^k \!,\! B^k\!}\!\! }\!\!} > \\
%& \hspace{25mm}\min \prl{\Big.\crl{\big.\linkage_A\prl{\vectbf{x}; A, C}, \linkage_A\prl{\vectbf{x}; A\sqcup C, B} } \setminus \crl{\big. \linkage_A\prl{\vectbf{x}; A, B}, \linkage_A\prl{\vectbf{x}; A\sqcup B, C}}},
%\end{align}
%} 
%%
%Thus, iterations of the online clustering does not cycle before reaching a locally structural optimal clustering hierarchy. 
%Moreover, the finite time convergence of the system is simply due to finiteness of tree space $\bintreetopspace_{\indexset}$,  which completes the proof.
\hfill \qed
\end{proof}

%%%%%%%%%%%%%%%%%%%%%%%%%%%%%%%%%%%%%%%%%%%%%%
%%%%%%%%%%%%%%%%%%%%%%%%%%%%%%%%%%%%%%%%%%%%%%
\section{Special Cases of  Average Linkage}
\label{app.SpecialAverageLinkage}
%%%%%%%%%%%%%%%%%%%%%%%%%%%%%%%%%%%%%%%%%%%%%%
%%%%%%%%%%%%%%%%%%%%%%%%%%%%%%%%%%%%%%%%%%%%%%

We now consider  certain settings of average linkage $\linkage_A$ \refeqn{eq.AverageLinkage}
that enables efficient computation of linkage values of a clustering hierarchy and its restructuring during online clustering.

Consider the squared Euclidean distance as a dissimilarity measure of a pair of data points, i.e. for any $\vect{x}, \vect{y} \in \R^m$
\begin{align}
\dist\prl{\vect{x}, \vect{y}} = \norm{\vect{x} - \vect{y}}_2^2.
\end{align}  
For any $\vectbf{x} \in \prl{\R^d}^{\indexset}$ and disjoint subsets $A,B \subseteq \indexset$, the average linkage $\linkage_{A}$ \refeqn{eq.AverageLinkage} between partial patterns $\vectbf{x}|A$ and $\vectbf{x}|B$, based on the squared Euclidean distance, is
{\small
\begin{align}
\linkage_A\prl{\vectbf{x}; A, B} &\sqz{=} \frac{1}{\card{A}\card{B}}\sum_{\substack{a \in A\\ b \in B}} \! \dist\prl{\vect{x}_a, \vect{x}_b} \sqz{=} \frac{1}{\card{A}\card{B}}\sum_{\substack{a \in A\\ b \in B}} \!\norm{\vect{x}_a \sqz{-} \vect{x}_b}_2^2\!,\!\! \\
&= \frac{1}{\card{A} \card{B}} \sum_{\substack{a \in A\\ b \in B}} \norm{\vect{x}_a - \ctrd{\vectbf{x}|B} + \ctrd{\vectbf{x}|B} - \vect{x}_b}_2^2. \label{eq.AverageLink1}
\end{align}
}
After expanding the norm, \refeqn{eq.AverageLink1} simplifies to
{\small
\begin{align}
\linkage_A\!\prl{\vectbf{x};\! A,\! B}&\sqz{=} \frac{1}{\card{A}} \! \sum_{a \in A} \! \norm{\vect{x}_a \sqz{-} \ctrd{\vectbf{x}|B}}_2^2 \sqz{+} \frac{1}{\card{B}} \! \sum_{\substack{ b \in B}} \! \norm{\vect{x}_b \sqz{-} \ctrd{\vectbf{x}|B}}_2^2\!,\!\!\! \\
&\hspace{-5mm}= \frac{1}{\card{A}} \! \sum_{a \in A} \norm{\vect{x}_a \sqz{-} \ctrd{\vectbf{x}|A} \sqz{+} \ctrd{\vectbf{x}|A} \sqz{-} \ctrd{\vectbf{x}|B}}_2^2 \sqz{+} \var{\vectbf{x}|B} \!. \!\!\! \label{eq.AverageLink2}
\end{align}
}
Using a similar trick on \refeqn{eq.AverageLink2}, one can conclude that the value of average linkage between $\vectbf{x}|A$ and $\vectbf{x}|B$ is a simple function of their centroids and variances, \footnote{This is generally referred to the ``bias-variance" decomposition of squared Euclidean distance \cite{hastie_tibshirani_friedman_2009}.} 
\begin{align} 
\linkage_A\!\prl{\vectbf{x};\! A,\! B} \sqz{=} \var{\vectbf{x}|A} + \var{\vectbf{x}|B} + \norm{\ctrd{\vectbf{x}|A} - \ctrd{\vectbf{x}|B}}_2^2.  \label{eq.AverageLinkSimple}
\end{align}
where for any $I \subseteq \indexset\;$ $\ctrd{\vectbf{x}|I}$\refeqn{eq.centroid} and $\var{\vectbf{x}|I}$\refeqn{eq.variance} denote the centroid and variance of $\vectbf{x}|I$, respectively.

A similar computational improvement for average linkage is also possible with the cosine dissimilarity --- another commonly used dissimilarity, in information retrieval and text mining \cite{salton_buckley_IPM1988}: for any $\vect{x}, \vect{y} \in \R^m$,
\begin{align}
\dist\prl{\vect{x}, \vect{y}} = 1 - \frac{\dotprod{\vect{x}}{\vect{y}}}{\norm{\vect{x}}_2 \norm{\vect{y}}_2},
\end{align}    
where  $\dotprod{}{}$  denote the Euclidean dot product.
For any dataset of unit length vectors $\vectbf{x} \in \prl{\Sp^{m-1}}^{\indexset}$ and disjoint subsets $A,B \subseteq \indexset$, the average linkage $\linkage_A$ \refeqn{eq.AverageLinkage} between $\vectbf{x}|A$ and $\vectbf{x}|B$, based on the cosine dissimilarity, can be rewritten as 
{\small
\begin{align}
\hspace{-2mm}\linkage_A\prl{\vectbf{x};\!A,\!B} &\sqz{=} \frac{1}{\card{A}\!\card{B}}\!\sum_{\substack{a \in A\\ b \in B}} \!\!\dist\prl{\vect{x}_a, \vect{x}_b} \sqz{=} 
\frac{1}{\card{A}\!\card{B}} \!\sum_{\substack{a \in A\\ b \in B}} \! 1 \sqz{-} \!\! \!\underbrace{\frac{\dotprod{\vect{x}_a}{\vect{x}_b}}{\norm{\vect{x}_a}_2 \!\norm{\vect{x}_b}_2}}_{ \norm{\vect{x}_a}_2 = \norm{\vect{x}_b}_2 = 1 }\!\!,\!\! \nonumber\\
&\sqz{=} 1 \sqz{-}  
\frac{1}{\card{A}\!\card{B}} \!\sum_{\substack{a \in A\\ b \in B}} \!  \dotprod{\vect{x}_a}{\vect{x}_b} 
%\sqz{=} 1 \sqz{-}\frac{1}{\card{A}} \sum_{a \in A} \dotprod{\vect{x}_a}{\underbrace{\prl{\frac{1}{\card{B}} \sum_{b \in B} \vect{x}_b}}_{= \ctrd{\vectbf{x}|B}}} \\
%&\sqz{=} 1- \dotprod{\underbrace{\prl{\frac{1}{\card{A}}\sum_{a \in A}\vect{x}_a}}_{=\ctrd{\vectbf{x}|A}}}{\ctrd{\vectbf{x}|B}}
\sqz{ = }  1 - \ctrd{\vectbf{x}|A} \cdot \ctrd{\vectbf{x}|B},
\end{align}   
}
which directly follows the linearity of the dot product.  

%%%%%%%%%%%%%%%%%%%%%%%%%%%%%%%%%%%%%%%%%%%%%%%%%%%
%%%%%%%%%%%%%%%%%%%%%%%%%%%%%%%%%%%%%%%%%%%%%%%%%%%
\section{Sample Variance and Mean \\ After Merging Clusters}
\label{app.VarianceMeanMerge}
%%%%%%%%%%%%%%%%%%%%%%%%%%%%%%%%%%%%%%%%%%%%%%%%%%%
%%%%%%%%%%%%%%%%%%%%%%%%%%%%%%%%%%%%%%%%%%%%%%%%%%%

It is well know that for any $\vectbf{x} \in \prl{\R^{d}}^{\indexset}$ and any disjoint subsets $A,B \subseteq \indexset$ of a finite index set $\indexset$, the centroid of merged patterns $\vectbf{x}|A \cup B$ is simply equal to the weighted average, proportional to the cardinality of sets, of centroids of partial patterns $\vectbf{x}|A$ and $\vectbf{x}|B$, 
\begin{align}\label{eq.CentroidMerge}
\ctrd{\vectbf{x}|A\cup B} = \frac{\card{A}}{\card{A}+\card{B}}\ctrd{\vectbf{x}|A} + \frac{\card{B}}{\card{A}+\card{B}}\ctrd{\vectbf{x}|B}.
\end{align}

Similarly, using \refeqn{eq.WardLinkageSSE}, one can verify that
for any disjoint subsets $A,B \subseteq \indexset$ and $\vectbf{x} \in \prl{\R^{d}}^{\indexset}$, the variance of merged patterns $\vectbf{x}|A\cup B$ is given by
{\small
\begin{align} \label{eq.VarianceMerge}
\var{\vectbf{x}|A\sqz{\cup} B} &\sqz{=} \frac{\card{A}}{\card{A}\sqz{+}\card{B}} \var{\vectbf{x}|A} \sqz{+} \frac{\card{B}}{\card{A}\sqz{+}\card{B}} \var{\vectbf{x}|B} \nonumber \\
& \hspace{15mm} \sqz{+} \frac{\card{A}\card{B}}{\prl{\card{A}\sqz{+}\card{B}}^2} \!\norm{ \big. \ctrd{\vectbf{x}|A} \sqz{-} \ctrd{\vectbf{x}|B}}_2^2. 
\end{align}
}

%\bigskip
%
%For any $\vectbf{x} \in \prl{\R^m}^{\indexset}$ and $I \subseteq \indexset$, let $\mat{C}\prl{\vectbf{x}| I} \in \R^{m \times m}$ denote the sample covariance of partial observation $\vectbf{x}|I$,
%%
%\begin{align}
%\mat{C}\prl{\vectbf{x}|I} \ldf \frac{1}{\card{I}} \sum_{i \in I} \prl{\big.\vect{x}_i - \ctrd{\vectbf{x}|I}} \tr{\prl{\big.\vect{x}_i - \ctrd{\vectbf{x}|I}}}.
%\end{align}
%%
%Now, one can simply verify that   for any disjoint subsets $A,B \subset \indexset$, the covariance matrix of merged cluster $\vectbf{x}|A \cup B$ can be written in terms of the sufficient statistics of $\vectbf{x}|A$ and $\vectbf{x}|B$ as follows:
%%
%{\small
%\begin{align}
%\hspace{-1mm}\mat{C}\prl{\vectbf{x}|A \sqz{\cup} B} &\sqz{=} \frac{\card{A}}{\card{A}+\card{B}} \mat{C}\prl{\vectbf{x}|A} + \frac{\card{B}}{\card{A}+\card{B}} \mat{C}\prl{\vectbf{x}|B} \nonumber \\
%& \hspace{-4mm}+ \frac{\card{A}\card{B}}{\prl{\card{A} \sqz{+} \card{B}}^2} \prl{\big.\ctrd{\vectbf{x}|A} \sqz{-} \ctrd{\vectbf{x}|B}\!}  \tr{\prl{\big.\ctrd{\vectbf{x}|A} - \ctrd{\vectbf{x}|B}\!}\!}\!\!.\!\!
%\end{align}
%}

\end{document}